\documentclass[sigconf]{acmart}
\renewcommand\footnotetextcopyrightpermission[1]{} 

\AtBeginDocument{%
  \providecommand\BibTeX{{%
    \normalfont B\kern-0.5em{\scshape i\kern-0.25em b}\kern-0.8em\TeX}}}

\setcopyright{acmcopyright}
\copyrightyear{2018}
\acmYear{2018}
\acmDOI{10.1145/1122445.1122456}

\acmConference[Woodstock '18]{Woodstock '18: ACM Symposium on Neural
  Gaze Detection}{June 03--05, 2018}{Woodstock, NY}
\acmBooktitle{Woodstock '18: ACM Symposium on Neural Gaze Detection,
  June 03--05, 2018, Woodstock, NY}
\acmPrice{15.00}
\acmISBN{978-1-4503-XXXX-X/18/06}

\usepackage{algorithm} 
\usepackage[T1]{fontenc}
\usepackage{algpseudocode}
\usepackage{multirow}

\graphicspath{{figs/}}

\newcommand{\wordtovec}{Word2Vec }
\newcommand{\wordtovecend}{Word2Vec}

\newcommand{\nodetovec}{Vertex2Vec }
\newcommand{\nodetovecend}{Vertex2Vec}

\newcommand{\anytovec}{Any2Vec }
\newcommand{\anytovecend}{Any2Vec}

\newcommand{\graphwordvec}{GraphWord2Vec }
\newcommand{\graphwordvecend}{GraphWord2Vec}

\newcommand{\graphnodevec}{GraphVertex2Vec }

\newcommand{\graphanyvec}{GraphAn\-y2Vec }
\newcommand{\graphanyvecend}{GraphAn\-y2Vec}

\newcommand{\gradshadow}{orthogonality }
\newcommand{\modelcombiner}{Gradient Combiner }

\newcommand{\real}{{\rm I\!R}}
\newcommand{\norm}[1]{\left\lVert#1\right\rVert}
\DeclareMathOperator*{\argmin}{arg\,min}

\begin{document}

\title{Distributed Training of Embeddings using Graph Analytics}

\author{Gurbinder Gill}
\email{gill@cs.utexas.edu}
\affiliation{%
  \institution{The University of Texas at Austin}
  \city{Austin}
  \state{Texas}
  \country{USA}
}

\author{Roshan Dathathri}
\email{roshan@cs.utexas.edu}
\affiliation{%
  \institution{The University of Texas at Austin}
  \city{Austin}
  \state{Texas}
  \country{USA}
}

\author{ Saeed Maleki}
\email{saemal@microsoft.com}
\affiliation{%
  \institution{Microsoft Research}
  \country{USA}
}

\author{Madan Musuvathi}
\email{madanm@microsoft.com}
\affiliation{%
  \institution{Microsoft Research}
  \country{USA}
}

\author{Todd Mytkowicz}
\email{toddm@microsoft.com}
\affiliation{%
  \institution{Microsoft Research}
  \country{USA}
}

\author{Olli Saarikivi}
\email{olsaarik@microsoft.com}
\affiliation{%
  \institution{Microsoft Research}
  \country{USA}
}

\begin{abstract}
  Many applications today, 
such as natural language processing, network analysis, and code analysis, 
rely on semantically embedding objects into low-dimensional
fixed-length vectors. Such embeddings naturally provide a way to perform useful
downstream tasks, such as identifying relations among objects or predicting
objects for a given context, etc. 
Unfortunately, the training necessary for
accurate embeddings is usually computationally intensive and requires processing
large amounts of data. 
Furthermore, distributing this training is challenging.
Most embedding training uses stochastic gradient
descent (SGD), an ``inherently'' sequential algorithm where at each
step, the processing of the current example depends on the
parameters learned from the previous examples.  Prior approaches to
parallelizing SGD do not honor these dependencies and thus
potentially suffer poor convergence.

This paper presents a distributed training framework for a class of applications
that use {\em Skip-gram-like} models 
to generate embeddings. 
We call this class \anytovec and it includes \wordtovecend,
DeepWalk, and Node2Vec
among others. 
We first formulate \anytovec training algorithm as a 
graph application and leverage the state-of-the-art
distributed graph analytics framework, D-Galois. 
We adapt D-Galois to support dynamic graph generation and
re-partitioning, and 
incorporate novel communication optimizations. 
Finally, we introduce a novel way to combine gradients 
during the distributed training to prevent accuracy loss. 
We show that our framework, called \graphanyvecend, 
matches on a cluster of 32 hosts the accuracy of 
the state-of-the-art shared-memory implementations 
of \wordtovec and \nodetovec on 1 host, 
and gives a geo-mean speedup of $12\times$ and $5\times$ 
respectively. 
Furthermore, \graphanyvec is on average $2\times$ faster than 
the state-of-the-art distributed \wordtovec implementation, 
DMTK,  
on 32 hosts. 
We also show the superiority 
of our {\em Gradient Combiner} independent of \graphanyvec by incorporating
it in DMTK, which raises its accuracy by $>30\%$. 

\end{abstract}

\maketitle

\section{Introduction}
\label{sec:intro}
Many applications today, such as natural language processing~\cite{word2vec1,
word2vec2, doc2vec}, network analysis~\cite{deepwalk,
node2vec,lBSN2vec}, and code analysis~\cite{Alon:2019:CLD:3302515.3290353,
code2seq}, rely on semantically embedding objects into low-dimensional
fixed-length vectors. Such embeddings naturally provide a way to perform useful
downstream tasks, such as identifying relations among objects or predicting
objects for a given context, etc. Unfortunately, the training necessary for
accurate embeddings is usually computationally intensive and requires processing
large amounts of data. 

This paper presents a distributed training framework for a class of applications
that use {\em Skip-gram-like} models, like the one used in
\wordtovecend~\cite{word2vec1}, to generate embeddings. We call this class
\anytovec and includes, in addition to \wordtovecend~\cite{word2vec1},
DeepWalk~\cite{deepwalk} and Node2Vec~\cite{node2vec}
among others. Applications in this class maintain a large embedding matrix,
where each row corresponds to the embedding for each object. Given sequences of
objects (text segments for \wordtovec and graph paths in DeepWalk), the training
involves looking up the embedding matrix for the objects in the sequence and
updating them through stochastic gradient descent~(SGD). The details of the how
the sequences are generated and the cost functions used to update the embeddings
varies with the application. 

The key challenge in distributing \anytovec training is that SGD is {\em
inherently} sequential. Two approaches for parallelizing SGD are {asynchronous}
SGD, where multiple nodes racily update~\cite{hogwild} a model that may be housed in a
global parameter server~\cite{40565}, or {synchronous} SGD, where nodes
bulk-synchronously combine individual gradients in a mini-batch update before
updating the model~\cite{JMLR:v15:agarwal14a}. It is well known that the
staleness of updates affects the scalability of the former, while the increase
in mini-batch size affects the scalability of the latter. This is substantiated
in our evaluation. 

To improve scalability over prior methods, this paper introduces
\graphanyvecend, a distributed machine learning framework for \anytovecend. We
first demonstrate that the \anytovec class of machine learning algorithms can be
formulated as a graph application and leverage the ease of programming and
scalability of the state-of-the-art distributed graph analytics frameworks, such
as D-Galois~\cite{gluon} and Gemini~\cite{gemini}. 
To support this new application, 
we extend D-Galois to support dynamic graph generation and
re-partitioning, and implement communication optimizations for reducing the
communication volume, the main bottleneck for these applications at scale.
Finally, we introduce a novel way to combine gradients during distributed
training to prevent accuracy loss when scaling. Rather than simply averaging the
gradients, as in a synchronous mini-batch SGD, our 
\modelcombiner (GC) performs a
weighted combination on gradients based on whether they are parallel or
orthogonal to each other.   
  

We evaluate two applications, \wordtovec and \nodetovecend, in our \graphanyvec framework 
on a cluster of up to 32 machines with 3 different datasets each. We 
compare \graphanyvec training time and
accuracy with the state-of-the-art shared-memory implementations
(original C implementation~\cite{word2vec2} and Gensim~\cite{gensim} for \wordtovec  
and DeepWalk~\cite{deepwalk} for \nodetovecend) as 
well as with the state-of-the-art distributed parameter-server \wordtovec
implementation in Microsoft's Distributed Machine Learning 
Toolkit (DMTK)~\cite{DMTK}. 
We show that compared to shared-memory implementations, \graphanyvec can 
reduce the training time for \wordtovec from 21 hours to less than 2 hours on our 
largest dataset of Wikipedia articles
while matching the SGD accuracy of shared-memory implementations, 
and gives a geo-mean 
speedup of $12\times$ and $5\times$ for \wordtovec and \nodetovec respectively.
On 32 hosts, \graphanyvec is on average $2\times$ faster 
than DMTK. We also show the superiority 
of our {\em Gradient Combiner} (GC) independent of \graphanyvec by incorporating
it in DMTK, which raises its accuracy by $>30\%$ 
so that it matches its own shared-memory implementation. 

The rest of this paper is organized as follows. 
Section~\ref{sec:background} provides a background 
on SGD, training of \anytovec and graph analytics. 
Section~\ref{sec:word2vec_graph} describes our \graphanyvec 
framework and 
Section~\ref{sec:combiner} describes our novel 
way to combine gradients. 
Section~\ref{sec:evaluation} presents our evaluation 
of \wordtovec and \nodetovec using \graphanyvecend. 
Related work and conclusions are presented in 
Sections~\ref{sec:related} and~\ref{sec:conclusions}.

\section{Background}
\label{sec:background}
In this section, we first briefly describe how stochastic gradient descent is
used to train machine learning models (Section~\ref{subsec:sgd}), followed by how Any2Vec models are
trained with Word2Vec as an example (Section~\ref{subsec:any2vec}).
We then provide an overview of graph analytics 
(Section~\ref{subsec:graph_analytics}).

\subsection{Stochastic Gradient Descent}
\label{subsec:sgd}

We express the training task of a machine learning model as a set of multivariable loss functions $L_i(w):\real^n\rightarrow \real$
where $w$ is the model and each $L_i$ corresponds to the training sample $i$. The output of $L_i(w)$ is a positive
value that correlates the prediction of the model $w$ to the label of sample $i$. Perfect prediction has a loss of 0.
The ultimate goal is to find $w$ that minimizes the loss function across all samples: $\argmin_w:\sum_i L_i(w)$.

Stochastic Gradient Descent (SGD)~\cite{bottou2012stochastic} is a popular algorithm for machine learning training.
The model is initially set to a random guess $w_0$ and 
at iteration or sample $i$, 
\[w_i:=w_{i-1} - \alpha \cdot \frac{\partial L_i}{\partial w}\bigg|_{w_{i-1}} \] where $\alpha$ is the {\it learning rate} 
and $\frac{\partial L_i}{\partial w}\big|_{w_{i-1}}$ is the {\it gradient} of
$L_i$ at $w_{i-1}$. Training is complete when the model reaches a desired loss
or evaluation accuracy. An \emph{epoch} of training is the number of updates
needed to go through the whole dataset once.

The fact that SGD's update rule for $w_i$ depends on $w_{i-1}$ makes SGD an
inherently sequential algorithm. A well-known technique to introduce parallelism
is {\em mini-batch} SGD~\cite{gdstudy}, wherein the gradient is calculated as an
average over $n$ training examples, where $n$ is the \emph{mini-batch size}.
When $n$ is 1 this is equivalent to normal SGD.


Hogwild!~\cite{hogwild} is another well-known SGD parallelization technique,
wherein multiple threads compute gradients for different training examples in
parallel and update the model in a racy fashion. Surprisingly, this approach
works well on a shared-memory system, especially with models where gradients are
sparse.

This paper uses intuitions based on the {\bf Taylor expansion of SGD} to develop
new techniques for parallelizing SGD. Applying the SGD update rule to a loss
function $L_i$ and expanding with the Taylor approximation gives:
\begin{equation} \label{eq:tesgd}
\begin{split}
L_i(w_i)&=L_i(w_{i-1}-\alpha\cdot g_i)\approx L_i(w_{i-1})-\alpha\cdot g_i^T \cdot \frac{\partial L_i}{\partial w}\bigg|_{w_{i-1}} \\
        &=L_i(w_{i-1})-\alpha\cdot g_i^T\cdot g_i \\
				&=L_i(w_{i-1})-\alpha\cdot \norm{g_i}^2
\end{split}
\end{equation}
As it is clear from Equation~\ref{eq:tesgd}, 
moving in the direction of the gradient reduces the loss. 
Note that the learning rate, $\alpha$, is a delicate hyper-parameter: a small $\alpha$ decays the loss insignificantly and for large $\alpha$,
the Taylor expansion approximation breaks and the model diverges.

\subsection{Training Any2Vec Embeddings}
\label{subsec:any2vec}
\begin{figure}
    \includegraphics[trim={2cm 0 0 0},width=0.47\textwidth]{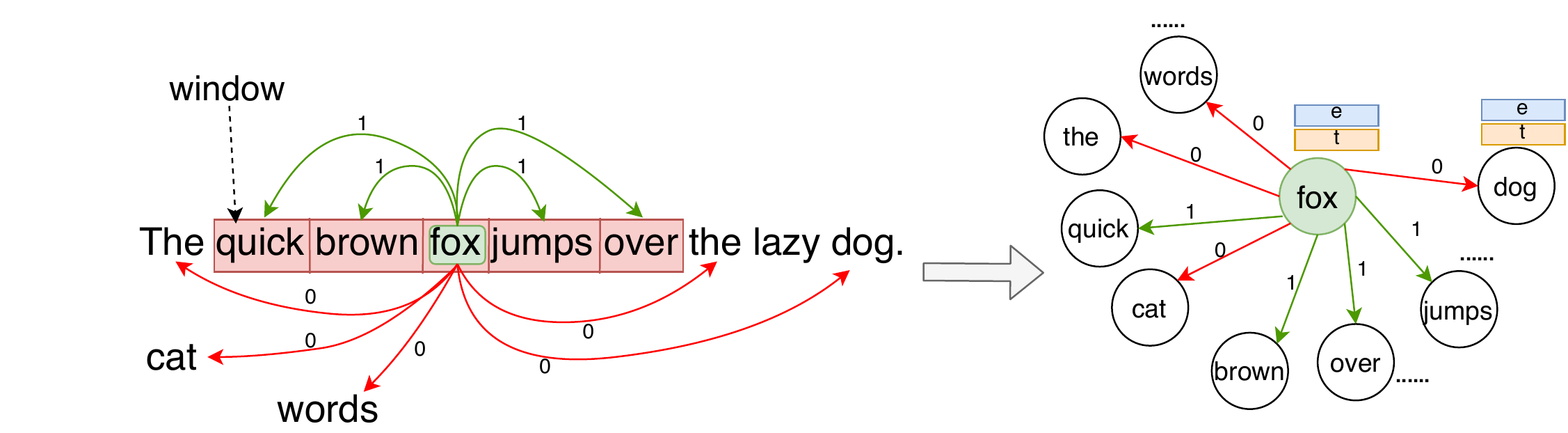}
    \caption{Viewing \wordtovec in Skip-gram model as a graph.}
    \label{fig:word2vec_basic}
\end{figure}


An embedding is a mapping from a dataset $D$ to a vector space $\real^N$ such
that elements of $D$ that are related are close to each other in the vector
space. The length $N$ of the \emph{embedding vectors} is typically much smaller
than the dimension of $D$.

Many models have been proposed for learning word
embeddings~\cite{word2vec1,word2vec2}. We will focus on the popular Skip-gram
model together with the negative sampling introduced in~\cite{word2vec2}, and
explain it here in a form suitable for a graph analytic understanding.

Skip-gram uses a training task where it predicts if a target word $w_O$
appears in the context of a center word $w_I$.
%
Figure~\ref{fig:word2vec_basic} (left) illustrates this for an example sentence,
with ``fox'' as the center word. The context
is then defined as the words inside a window of size $c$ (a hyper-parameter)
centered on ``fox''. In Figure~\ref{fig:word2vec_basic} these \emph{positive}
samples are shown in green and have a label of 1. For each positive sample
Skip-gram picks $k$ (a hyper-parameter) random words as \emph{negative} samples
and gives them a label of 0.


The Skip-gram model consists of two vectors of size $N$ for each word $w$ in the
vocabulary: an {\em embedding} vector $e_w$ and a {\em training} vector $t_w$.
For a pair of words the model predicts the label with $\sigma(e_{w_I}^T\cdot
t_{w_O})$, which should be close to 1 for related words and close to 0
otherwise. The loss term for a sample is then $-\log (1-|y-\sigma(e_{w_I}^T\cdot
t_{w_O})|)$, where $y$ is the true label of the sample.

We base the work in this paper on Google's \wordtovec
tool\footnote{https://code.google.com/archive/p/word2vec/}, which uses the
Hogwild! parallelization technique. Each thread is given a subset of the corpus
and goes through the words in it sequentially (skipping some due to sub-sampling
frequent words as described in~\cite{word2vec2}). For each pair of a central word and a target
word in its context \wordtovec calculates a gradient using a sum of the loss
term for the positive sample itself and loss terms for $k$ negative
samples. This gradient is then applied to the model shared by all threads in a
racy manner and the thread continues onto the next pair of words.

\subsection{Graph Analytics}
\label{subsec:graph_analytics}

In typical graph analytics applications, 
each node has one or more labels, which are 
updated during algorithm execution 
until a global quiescence condition is reached. 
The labels are updated by iteratively applying a computation rule, known as an {\it operator},
to the nodes or edges in the graph. 
The order in which the operator is applied to the nodes or edges 
is known as the {\it schedule}.
A {\em node operator} takes a node $n$ and updates labels
on $n$ or its neighbors, whereas an {\em edge operator} takes an edge $e$ and updates labels
on source and destination of $e$.


To execute graph applications in distributed-memory, 
the edges are first partitioned~\cite{cusp} among the hosts
and for each edge on a host, 
proxies are created for its endpoints. 
As a consequence of this design, a node might have proxies (or replicas) on many hosts.
One of these is chosen as the {\it master} proxy to hold the canonical
value of the node. The others are known as {\it mirror} proxies. 
Several heuristics exist for partitioning edges and choosing {\it master}
proxies~\cite{partitioningstudy}.

Most distributed graph analytics systems~\cite{powergraph,gemini,gluon} 
use bulk-synchronous parallel (BSP) execution.
Execution is done in rounds 
of computation followed by bulk-synchronous communication.
In the computation phase, every host applies the {\it operator} 
inside its own partition and updates the labels of the local proxies. 
Thus, different proxies of the same node might have different values.
Every host then participates in a global communication phase 
to synchronize the labels of all proxies. 
Different proxies of the same node are reconciled 
by applying a {\it reduction} operator, 
which depends on the algorithm being executed. 

\section{Distributed \anytovecend}
\label{sec:word2vec_graph}

In this section, we first describe the formulation of \anytovec as a graph application 
and provide an overview of our distributed \graphanyvec (Section~\ref{sec:overview}). We then describe the different 
phases in our approach such as dynamic graph generation and partitioning 
(Section~\ref{sub:graph-construction}), 
model synchronization (Section~\ref{sec:model-sync}), and communication optimizations (Section~\ref{sec:comm-opt}).

\subsection{Overview of Distributed \graphanyvecend}
\label{sec:overview}

We formulate \anytovec as a graph problem and 
call it \graphanyvecend.
Each element in the dataset 
corresponds to a node in a graph, 
and the positive and negative samples correspond to  
edges in the graph with weights 1 and 0 respectively. 
Figure~\ref{fig:word2vec_basic} (right) illustrates this 
for \wordtovecend. 
Training the Skip-gram model is now a graph analytics application.
Each node has two labels --- $e$ and $t$ ---
for embedding and training vectors,
respectively, of size $N$. 
The model corresponds to these labels for all nodes.
These labels are initialized randomly and 
updated during training by applying 
an \emph{edge operator}, that 
takes the source $src$ and destination $dst$ of an edge 
with weight $w$, 
computes $\sigma(e_{\texttt{src}}^T\cdot t_{\texttt{dst}})$ to 
predict the relation between the two nodes, 
and then applies the SGD update rule to 
$e_{\texttt{src}}$ and $t_{\texttt{dst}}$ 
so as to minimize the loss function $-\log (1-|w-\sigma(e_{\texttt{src}}^T\cdot
t_{\texttt{dst}})|)$.
The operator is applied to all edges once in each epoch.

\begin{algorithm}[t]
  \small
  \caption{Execution on each distributed host of \graphanyvecend.}
  \label{alg:execution}
  \begin{algorithmic}[1]
    \Procedure{\graphanyvecend}{Corpus $C$, Num. of epochs $R$, Num. of sync rounds $S$, Learning rate $\alpha$}
    \State Let $h$ be the host ID
    \State Stream $C$ from disk to build set of vertices $V$
    \State Read partition $h$ of $C$ that forms the work-list of vertices $W$
    \For{epoch $r$ from 1 to $R$}
    \For{sync round $s$ from 1 to $S$}
    \State Let $W_s$ be partition $s$ of $W$
    \State Build graph $G$ = ($V$, $E$) where $E$ are samples in $W_s$
    \State Compute($G$, $W_s$, $\alpha$) \Comment{Updates $G$ and decays $\alpha$}
    \State Synchronize($G$) \Comment{Updates $G$}
    \EndFor
    \EndFor
    \EndProcedure
  \end{algorithmic}
\end{algorithm}


Algorithm~\ref{alg:execution} gives a brief overview of our distributed \graphanyvec execution.
The first step is to construct the set of vertices $V$ (unique words in case of \wordtovecend)
by making a pass over 
the training data corpus $C$ on each host in parallel. 
As $C$
may not fit in the memory of a single host, we stream it from disk to construct $V$.
%
The corpus $C$ is then partitioned 
(logically) into roughly equal contiguous chunks among 
hosts. All hosts read their own partition of $C$ in parallel.
The list of elements in a given host's partition of $C$ 
constitutes the work-list\footnote{The work-list $W$
does not change across epochs, so we construct it once and reuse it 
for all epochs and synchronization rounds. 
However, if it does not fit in memory, partition $s$ of $W$ 
can be constructed from the corpus in each synchronization round.} $W$
that the host is responsible for computing \anytovec on.
We introduce a new parameter 
for controlling the number of synchronization rounds 
within an epoch.
In each epoch on each host, 
$W$ is partitioned into roughly equal contiguous chunks among the rounds. 
In each round $s$, 
positive and negative samples from partition $s$ of $W$ 
are used to construct the graph. 
The \anytovec {\it operator} is then applied to all edges in the graph.
The operator updates the vertex labels directly 
and decays the learning rate continuously, 
as in shared-memory implementation of \anytovec applications.
Then, all hosts participate in a bulk-synchronous communication 
to synchronize the vertex labels.

We implement \graphanyvec in D-Galois~\cite{gluon}, the state-of-the-art distributed graph analytics framework, which 
consists of the Galois~\cite{galois} multi-threaded library for computation and 
the Gluon~\cite{gluon} communication substrate 
for synchronization.
Galois provides efficient, concurrent data structures like 
graphs, work-lists, dynamic bit-vectors, etc., 
which makes it quite straightforward to implement \graphanyvecend. 
Gluon incorporates communication optimizations 
that enable it to scale to a large number of hosts.
However, D-Galois only works with static graphs 
(nodes and edges must not change during algorithm execution), 
whereas, for \anytovec applications, edges are sampled randomly and generated. 
We adapted D-Galois to handle dynamic graph generation efficiently during 
computation and communication, 
as explained in Sections~\ref{sub:graph-construction} and~\ref{sec:comm-opt} 
respectively. 
Our techniques can be used to modify other 
distributed graph analytics frameworks and implement \graphanyvec in them.
%

\subsection{Graph Generation and Partitioning} 
\label{sub:graph-construction}



As explained in Section~\ref{subsec:any2vec}, 
the Skip-gram model generates positive and negative samples 
using randomization. 
Consequently, the samples or edges generated for the same element or node in the corpus 
in different epochs may be different. 
As the same edge may not be generated again, 
one way to abstract this is to consider that 
the edges are being streamed and each edge is processed only once, 
even across epochs. 
Due to this, the graph needs to be constructed in each synchronization round, 
as shown in Algorithm~\ref{alg:execution}.

The graph can be explicitly constructed in each round. 
However, this may add unnecessary overheads as each edge is processed only once 
before the graph is destroyed. 
More importantly, this does not distinguish between 
edges (samples) from different occurrences of the same node (element) in the corpus. 
Consequently, the relative ordering of the edges from different nodes is not 
preserved. 
We observed that the accuracy of the model is highly sensitive 
to the order in which the edges are processed 
because the learning rate decays after each occurrence of the node is processed. 
Hence, the key to our graph formulation is that on each host,
{\it the schedule of applying operators on edges in \graphanyvec
must match the order in which samples would be processed in \anytovecend}. 
Note that the work-list $W$ preserves the ordering of element occurrences in the corpus.
Thus, \graphanyvec generates or streams edges on-the-fly using partition $s$ of $W$ in round $s$, 
instead of constructing the graph.

Each host generates edges for its own partition of the graph in each 
synchronization round. In other words, 
the graph is re-partitioned in every round. 
By design, each edge is assigned to a unique host. 
As mentioned in Section~\ref{subsec:graph_analytics}, 
node proxies are created for the endpoints of edges on a host. 
The {\it master} proxy for each node can be chosen 
from among its proxies thus created, 
but this would incur overheads in every round. 
We instead (logically) partition the nodes 
once into roughly equal contiguous chunks among the hosts 
and each host creates master proxies for the nodes in its partition. 
Proxies for other nodes on the host would be 
{\it mirror} proxies. 
Each mirror knows the host that has its master 
using the partitioning of nodes. 
Each master also needs to know the hosts with its mirrors.
We provide two ways to do this: RepModel and PullModel.

In {\bf RepModel}, each host has proxies for all nodes, 
so the entire model is replicated on each host. Thus, 
each host statically knows that every other host has mirror proxies 
for the masters on it. 
This allows \graphanyvec to assume that an edge 
between any two nodes can be generated on any host.  
In {\bf PullModel}, 
each hosts makes an inspection pass over $W$ 
before computation in each round to 
generate edges and track the nodes that would be accessed 
during computation. 
Mirror proxies are then created for the nodes tracked. 
For each mirror proxy created, the host communicates 
to the host that has its master (bulk-synchronization).

RepModel requires the entire model to fit in the memory of a host, 
while PullModel enables handling larger models.
On the other hand, PullModel incurs overhead for determining 
masters and mirrors in each round, whereas 
RepModel does not. 
Nonetheless, 
RepModel and PullModel require different communication 
to synchronize the model, 
so we evaluate which of them performs better 
in Section~\ref{sec:evaluation}.

\begin{figure}
  \centering
    \includegraphics[width=.5\textwidth]{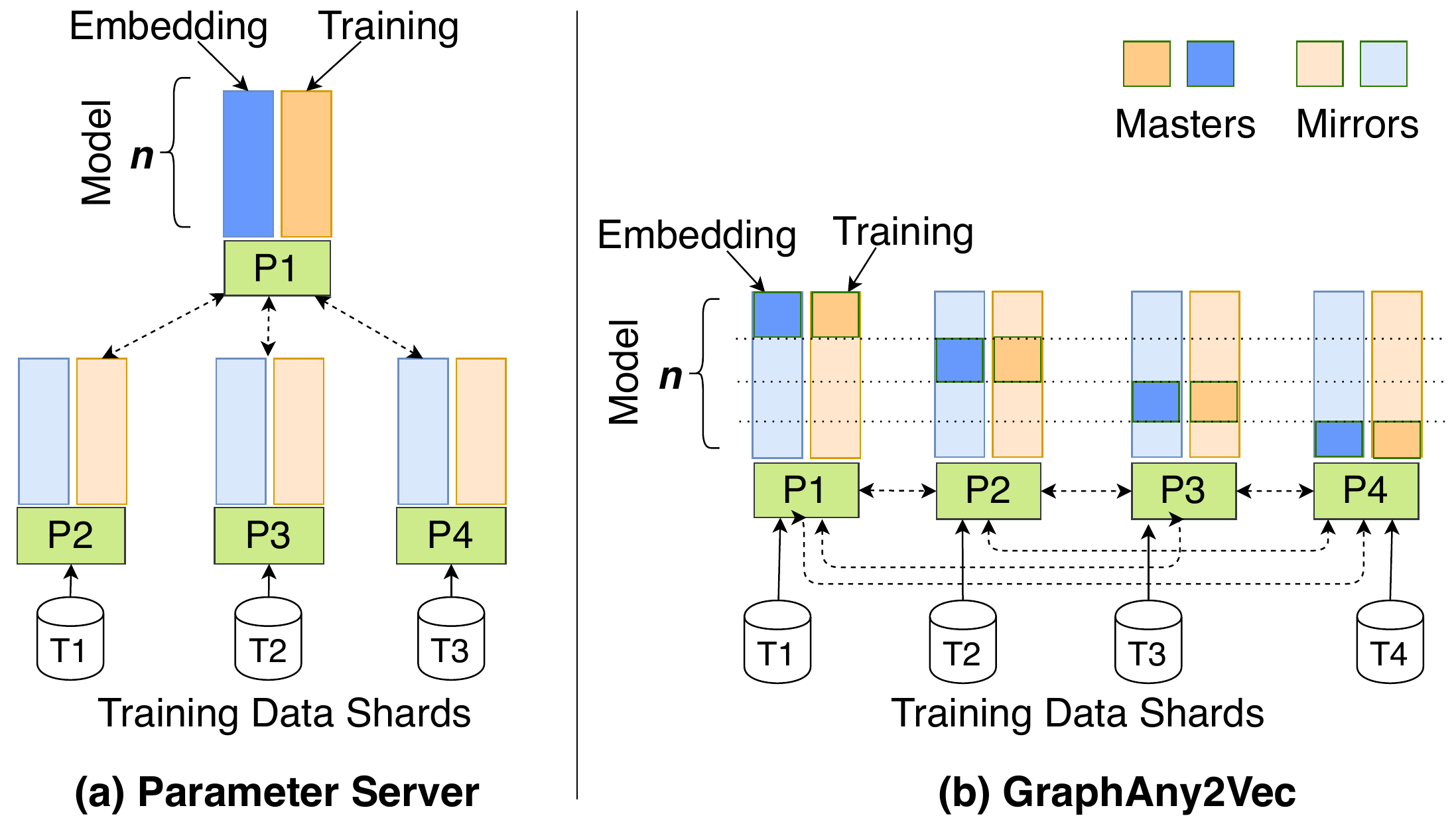}
\caption{Synchronization of the Skip-gram Model.}
\label{fig:model-sync}
\end{figure}



\begin{figure}[t]
  \centering
    \includegraphics[width=0.48\textwidth]{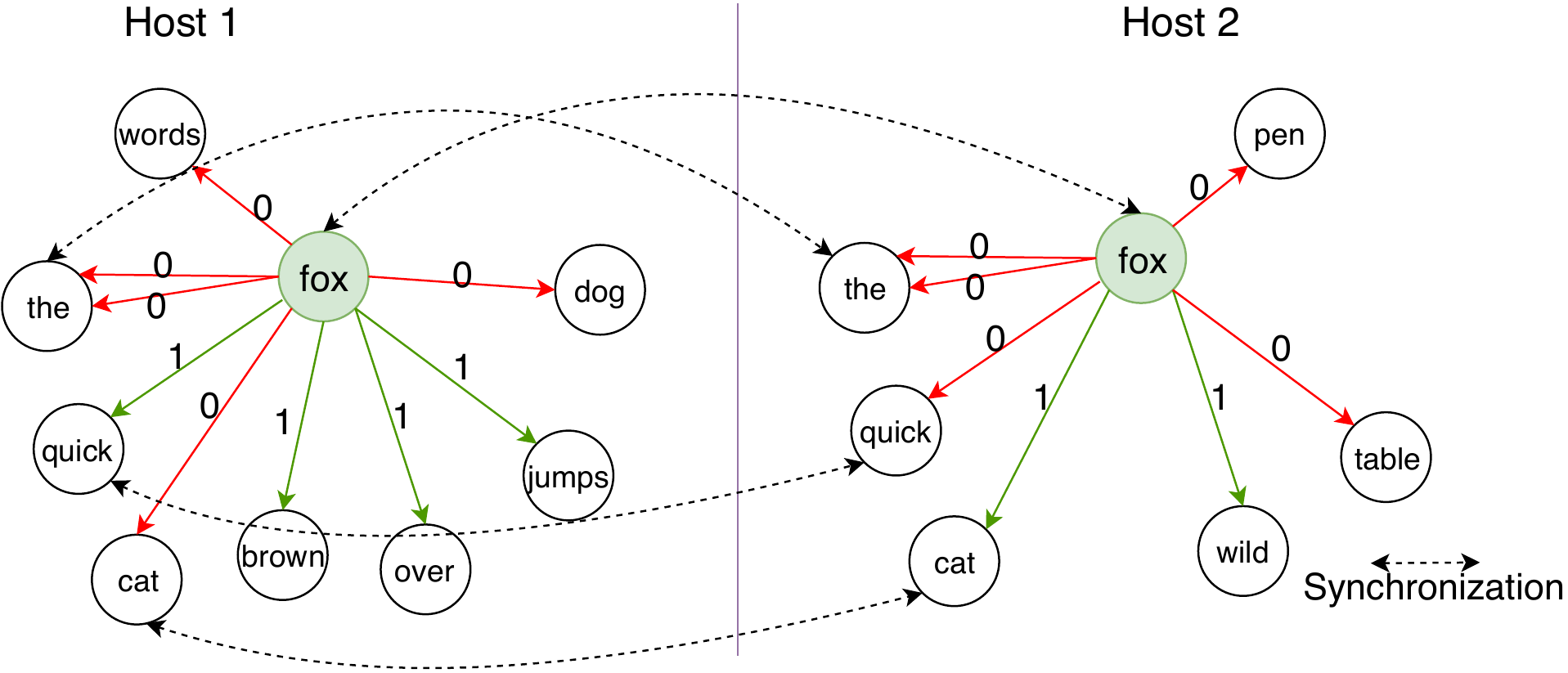}
\caption{Synchronization of proxies in graph partitions.}
\label{fig:word2vec_sync}
\end{figure}

\subsection{Model Synchronization}
\label{sec:model-sync}

Prior work for distributed \wordtovec such as 
Microsoft's Distributed Machine Learning Toolkit (DMTK)~\cite{DMTK} (and other machine learning algorithms) 
use a parameter server to synchronize the model, 
as illustrated in Figure~\ref{fig:model-sync}(a).
One of the hosts (say $P1$) is chosen as the parameter server.
At the beginning of a round (or a mini-batch), 
every host receives the updated model from the parameter server. 
The host then computes that round and sends the model updates to the parameter server.
\graphanyvec uses a different synchronization model based on D-Galois~\cite{gluon}, 
as illustrated in Figure~\ref{fig:model-sync}(b). 
Abstractly, this can be viewed as a generalization of the parameter server model 
where each host acts as a parameter server for a partition of the model.
%
In Figure~\ref{fig:model-sync}(b), 
$P1$ has the master proxies for the first contiguous chunk or partition of the nodes, 
$P2$ has the master proxies for the second partition of the nodes, and so on.
During synchronization in D-Galois, 
the mirror proxies send their updated value to the host containing the master, 
which reduces it and broadcasts it to the hosts containing the mirrors. 

Figure~\ref{fig:word2vec_sync} shows an example 
where proxies on two hosts need to be synchronized after computation. 
Consider the word "{\em fox}" that is present on both hosts. 
It may have different values for the embeddings on both hosts after computation.
The reduction operator determines how to synchronize these values 
and this is a parameter to synchronization in D-Galois.
As described in Section~\ref{sec:combiner}, 
averaging or adding the two values may lead to slower convergence, 
so we introduce a novel way to combine them called 
Gradient Combiner. 

\subsection{Communication Optimizations}
\label{sec:comm-opt}

{\bf RepModel-Naive}:
During synchronization in RepModel,
all mirrors on each host 
can send their updates to their respective masters and 
the masters can reduce those values and broadcast it to their mirrors. 
This is similar to communication for dense matrix codes, 
so can be mapped quite efficiently to MPI collectives. 
However, in \anytovecend, not all nodes are updated in every round. 
Consequently, such naive communication would result in 
redundant communication during both reduce and broadcast phases.

{\bf RepModel-Opt}:
The advantage of D-Galois is that it allows the user to specify the updated nodes 
and it would transparently handle the sparse communication that would entail. 
To do this, we maintain a bit-vector that tracks the nodes that were updated in this round. 
During synchronization, only the updated mirrors are sent to their masters 
and the masters broadcast their values to the other hosts only if it was updated on any host in that round. 
This avoids redundant communication during the reduce phase. 
However, there is still some redundancy during the broadcast phase 
because the update sent to a mirror might not be accessed by the mirror in the next round. 
This information remains unknown in RepModel. 

{\bf PullModel-Base}: 
In PullModel, 
mirrors are created after inspection only if 
one of its labels will be accessed on that host.
During synchronization, only the mirrors updated in 
this round are sent to their masters
(similar to RepModel-Opt). 
However, we wait to broadcast after inspection of the next round 
when new mirrors are created (re-partitioning).
During broadcast, all masters must be broadcast whether updated or
not, because previous updates may not have been sent to a host if it
did not have a mirror during a previous round.
This is essentially pulling the model that will be accessed (like in parameter server).
While this avoids sending masters to mirrors that do not access it, 
it may resend masters that have have been updated. 

{\bf PullModel-Opt}: 
Recall from Section~\ref{subsec:any2vec} that 
embedding vectors $e$ are accessed only at the source and 
training vectors $t$ are accessed only at the destination 
of an edge. 
If a mirror proxy on a host has only outgoing (or incoming) edges, 
then it will not access $t$ (or $e$). 
This is not exploited in PullModel-Base because 
masters and mirrors are not label-specific in D-Galois. 
We modified D-Galois to maintain masters and mirrors specific 
to each label. 
We also modified our inspection phase to track sources and 
destinations separately, and 
create mirrors for $e$ and $t$ respectively. 
Due to this, masters will broadcast $e$ and $t$ only to 
those hosts that access each.

\section{\modelcombiner}
\label{sec:combiner}

Section~\ref{sec:background} discussed how in the mini-batch approach gradients from 
multiple training examples are computed in parallel and they are reduced to a single 
vector by averaging. Although this is a widely-used practice, it does not follow
the semantics of the sequential algorithm.  
Suppose $L_1(w)$ and $L_2(w)$ are two loss functions corresponding
to two training examples. Starting from model $w_0$, sequential
SGD calculates $w_1=w_0 - \alpha \frac{\partial L_1}{\partial w}|_{w_0}$
followed by $w_2=w_1-\alpha \frac{\partial L_2}{\partial w}|_{w_1}$
where $\alpha$ is a proper learning rate. With forward substituition,
$w_2=w_0-\alpha (\frac{\partial L_2}{\partial w}|_{w_1}+\frac{\partial L_1}{\partial w}|_{w_0})$. 
Alternatively, in a parallel
setting, $\frac{\partial L_1}{\partial w}|_{w_0}$ and 
$\frac{\partial L_2}{\partial w}|_{w_0}$ (note that gradients are both
at $w_0$) are computed and $w$ is updated with $w_2'=w_0
-\frac{\alpha}{2}(\frac{\partial L_1}{\partial w}|_{w_0}
+\frac{\partial L_2}{\partial w}|_{w_0})$. Clearly $w_2'$ and $w_2$
are different because of the averaging effect. \cite{directsum} and~\cite{sqrtlr}
have claimed that scaling up the learning rate by the number of 
parallel processors (or square root of it) closes this gap. However, if
$\frac{\partial L_1}{\partial w}|_{w_0}$ and $\frac{\partial L_2}{\partial w}|_{w_0}$
are both in the same direction, scaling up the learning rate might cause divergence
as we assumed $\alpha$ was properly set for the sequential algorithm. Our {\em \modelcombiner}
addresses this problem by adjusting the gradients to each other.

Following the Taylor expansion for $\frac{\partial L_2}{\partial w}|_{w_0+(w_1-w_0)}$,
we have:
\begin{equation}\label{eq:te}
\begin{split}
\frac{\partial L_2}{\partial w}\Big|_{w_0+(w_1-w_0)} &\approx \frac{\partial L_2}{\partial w}\Big|_{w_0} 
+\frac{\partial^2 L_2}{(\partial w)^2}\Big|_{w_0}\cdot (w_1-w_0) \\
&=\frac{\partial L_2}{\partial w}\Big|_{w_0}-\alpha \frac{\partial^2 L_2}{(\partial w)^2}\Big|_{w_0} \cdot \frac{\partial L_1}{\partial w}\Big|_{w_0}
\end{split}
\end{equation}
where the approximation error is $O(\norm{w_1-w_0}^2)$ which is alternatively
$\alpha^2O\Big(\norm{\frac{\partial L_1}{\partial w}|_{w_0}}^2\Big)$. As
the learning rate $\alpha$ gets smaller, the error in Formula~\ref{eq:te} shrinks quadratically. Usually as the training 
of a \anytovec model progresses, the learning rate is decayed and as a result this error becomes negligible. 
For the rest of this section, we denote gradients $g_1=\frac{\partial L_1}{\partial w}|_{w_0}$,
$g_2=\frac{\partial L_2}{\partial w}|_{w_0}$, $g_2'=\frac{\partial L_2}{\partial w}|_{w_1}$ and the Hessian matrix $H_2=\frac{\partial^2 L_2}{(\partial w)^2}|_{w_0}$.
Therefore, Equation~\ref{eq:te} can be re-written by: 
\begin{equation}\label{eq:tes}
g_2'\approx g_2-\alpha H_2\cdot g_1
\end{equation}

Equation~\ref{eq:tes} lets us compute $g_2'$, however, computing $H_2$ is expensive as it is a $n\times n$ matrix
where $n$ is the number of parameters in the \anytovec model. Luckily because \anytovec has a log-likelihood loss
function, $H_2$
can be expressed by the outer product of the gradient: $\lambda g_2\cdot g_2^T$ where $\lambda$ is a scalar
which depends on $w_0$~\cite{msra}. The error for this approximation gets smaller as $w_0\rightarrow w^*$ where $w^*$ is
the optimal model parameters~\cite{msra,ggt}. By using Equation~\ref{eq:tes} and this approximation, $g_2'$ can be approximated by:
\begin{equation} \label{eq:teggt}
g_2'\approx g_2-\alpha \lambda g_2\cdot g_2^T\cdot g_1
\end{equation}
Although Formula~\ref{eq:teggt} makes calculation of $g_2'$ feasible, finding the right $\lambda$ for every
iteration of SGD is overwhelmingly difficult and it is yet another hyper-parameter for the user to tune. However,
if $g_1$ was orthogonal to $g_2$, then $g_2'$ could have been easily estimated by $g_2$. This is the intuition behind \modelcombiner.

Given that $g_1$ and $g_2$ are not always orthogonal, we project $g_1$ on the orthogonal space of $g_2$ to
make $g_1^O$:
\begin{equation}\label{eq:orth}
g_1^O=g_1-\frac{g_2^T\cdot g_1}{\norm{g_2}^2}g_2
\end{equation}
$g_1^O$ has three important properties: (1) $g_1^T\cdot g_1^O \geq 0$, (2) $\norm{g_1^O}\leq\norm{g_1}$,
and (3) $g_2^T\cdot g_1^O=0$. It is straight forward to check these properties 
(see Appendix~\ref{sec:prop} for full proof). Suppose $w_1^O=w_0 - \alpha g_1^O$. 
Then by using the Taylor expansion, we have:
\begin{equation}\label{eq:dl}
L_1(w_1^O) = L_1(w_0-\alpha g_1^O) \approx L_1(w_0) - \alpha g_1^T \cdot g_1^O \leq L_1(w_0)
\end{equation}
where the last inequality comes from property (1). This means that
moving in the direction of $g_1^O$ decays the loss of $L_1$. Also, because of property (2) and the fact
that the same learning rate as sequential learning rate is used, the approximation in Equation~\ref{eq:dl}
has the same or lower error as the one with $L_1(w_1)$ if we had the Taylor expansion for it (refer to Equation~\ref{eq:tesgd}). 
Property (3) and Equation~\ref{eq:teggt} ensure that 
$\frac{\partial L_2}{\partial w}|_{w_1^O}\approx \frac{\partial L_2}{\partial w}|_{w_0}=g_2$. Let's assume that sequential
SGD uses $g_1^O$ to get to $w_1^O=w_0-\alpha g_1^O$ followed by computing $\frac{\partial L_2}{\partial w}|_{w_1^O}$ which
can be approximated by $g_2$ to
get to $w_2^O=w_1^O-\alpha g_2$. By forward substituition:
\begin{equation}\label{eq:sgdcombiner}
w_2^O=w_0 - \alpha (g_2+g_1^O) = w_0 - \alpha \Bigg(g_2 + g_1-\frac{g_2^T\cdot g_1}{\norm{g_2}^2}g_2\Bigg)
\end{equation}
Therefore, the direction \modelcombiner ($GC(g_1,g_2)$ for short) 
uses to move is:
\begin{equation}\label{eq:gc}
GC(g_1,g_2)=g_2 + g_1-\frac{g_2^T\cdot g_1}{\norm{g_2}^2}g_2
\end{equation}	
which 
allows computation of $g_1$ and $g_2$ in parallel. Note that our gradient combiner requires slightly more computation than
averaging but as we will show in Section~\ref{sec:evaluation}, this overhead is negligible.

The magnitude of $GC(g_1,g_2)$ depends on how parallel or orthogonal $g_1$ and $g_2$ are to each other.
In the parallel case, $g_1^O$ becomes smaller and 
therefore, moving in the direction of $g_1^O$ decays the loss value of $L_1$ slower than $g_1$. 
However, as we will show in Section~\ref{sec:evaluation}, the gradients all start parallel to each other
in the begining of the training as they all point in the same general direction and later in the training
they become more orthogonal. This means that \modelcombiner conservatively takes small steps in the begining of
the training and larger ones as the training progresses.

\modelcombiner extends to combining $k$ gradients as well where gradients are combined sequentially manner as 
discussed in Section~\ref{sec:model-sync} ($GC(g_1,\dots,g_k)=GC(g_k,\dots,GC(g_3,GC(g_2,g_1))\dots)$). 
Appendix~\ref{sec:combiner-proof} proves the convergence of \modelcombiner in expectation. 


The effectiveness of \modelcombiner depends on the degree to which the gradients are orthogonal. 
We define 
\begin{equation}\label{eq:gradshadow}
O(g_1,g_2)=\frac{\norm{GC(g_1,g_2)^2}}{\norm{g_1}^2+\norm{g_2}^2}=\frac{\norm{g_1^O+g_2}^2}{\norm{g_1}^2+\norm{g_2}^2}
\end{equation}
as a notion for \gradshadow of $g_1$ and $g_2$. Note that because of property (3) $g_1^O$ and $g_2$ are orthogonal and
$\norm{g_1^O+g_2}^2 =\norm{g_1^O}^2+\norm{g_2}^2$ thanks to Pythagorean theorem. Therefore,
$\norm{g_1^O+g_2}^2 =\norm{g_1^O}^2+\norm{g_2}^2\leq \norm{g_1}^2+\norm{g_2}^2$ which concludes that
$t\leq 1$ and equality is met only when $g_1$ and $g_2$ are orthogonal. On the other hand, if $g_1=g_2$,
$g_1^O$ becomes zero and $O(g_1,g_2)=\frac{1}{2}$. This can be similarly expanded to $k$ gradients
$O(g_1,\dots,g_k)=\frac{GC(g_1,\dots,g_k)^2}{\sum_i\norm{g_i}^2}$. Similarly for $k$ gradients,
\gradshadow is $1$ when they are all orthogonal and $\frac{1}{k}$ when they are all the same. 


\section{Evaluation}
\label{sec:evaluation}

\begin{table}[t]
  \footnotesize
  \centering

  \caption{Datasets and their properties.}
  \label{tbl:datasets}
  \begin{tabular}{l@{\hskip 8pt}l@{\hskip 5pt}l@{\hskip 5pt}l@{\hskip 5pt}l@{\hskip 10pt}l}
  \toprule
  \multicolumn{1}{l}{\begin{tabular}[l]{@{}l@{}}\textbf{Applications}\end{tabular}} & \multicolumn{1}{l}{\begin{tabular}[l]{@{}l@{}}\textbf{Datasets}\end{tabular}}  &{\begin{tabular}[l]{@{}l@{}}\textbf{Vocabulary}\\\textbf{Words}\end{tabular}}  & \multicolumn{1}{l}{\begin{tabular}[l]{@{}l@{}}\textbf{Training}\\\textbf{Words}\end{tabular}} & \multicolumn{1}{l}{\begin{tabular}[l]{@{}l@{}}\textbf{Size}\end{tabular}}& \multicolumn{1}{l}{\begin{tabular}[l]{@{\hskip -10pt}l@{}}\textbf{Labels}\end{tabular}} \\
  \midrule
    \multirow{3}{*}{{\begin{tabular}[c]{@{}c@{}}\textbf{Word2Vec}\\\textbf{(Text)}\end{tabular}}}    & \textbf{1-billion} & 399.0K              & 665.5M                        & 3.7GB & N/A        \\
                                          & \textbf{news}                              & 479.3K              & 714.1M                        & 3.9GB & N/A        \\
                                          &  \textbf{wiki}                 & 2759.5K             & 3594.1M                       & 21GB     & N/A     \\ 

  \midrule
    \multirow{3}{*}{{\begin{tabular}[c]{@{}c@{}}\textbf{\nodetovec}\\\textbf{(Graph)}\end{tabular}}}    & \textbf{BlogCatalog}     & 10.3K             & 4.1M                        & 0.02GB      & 39   \\
                                          & \textbf{Flickr}     & 80.5K             & 32.2M                        & 0.18GB         & 195   \\  
                                          & \textbf{Youtube}   & 1138.5K           &455.4M                        & 2.8GB           & 47     \\

  \bottomrule
  \end{tabular}
\end{table}


  \begin{table}[t]
    \footnotesize
   \caption{\wordtovec training time (hours) on a single host.}
   \label{tbl:all_time}
    \centering
  \begin{tabular}{l|r|r|r|r}
   
       \toprule
       \multicolumn{1}{c|}{\textbf{Dataset}}   & \multicolumn{1}{c|}{\textbf{W2V}} & \multicolumn{1}{c|}{\textbf{GEM}}& \multicolumn{1}{c|}{\textbf{DMTK(AVG)}} & \multicolumn{1}{c}{\textbf{GW2V(GC)}} \\
     \midrule
     \multirow{1}{*}{\textbf{1-billion}} &   4.24     &  4.39 & 4.21  & 3.98      \\
     \multirow{1}{*}{\textbf{news}}      &   4.45     &  4.66 & 4.28  & 4.51      \\
     \multirow{1}{*}{\textbf{wiki}}      &   20.49     &  OOM     & 25.43  & 22.34      \\
       \bottomrule
     \end{tabular}
     \end{table}

  \begin{table}[t]
    \footnotesize
   \centering
   \caption{\wordtovec accuracy (semantic, syntactic, and total) of different frameworks relative to W2V.}
   \label{tbl:3rdparty_acc}
  \begin{tabular}{@{ }l@{\hskip1.8pt}|r@{\hskip5pt}r@{\hskip5pt}r@{\hskip5pt}r@{ }}
       \toprule
       & \multicolumn{1}{c|}{\textbf{Framework}}         & \multicolumn{1}{c|}{\textbf{1-billion}} & \multicolumn{1}{|c|}{\textbf{news}} & \multicolumn{1}{|c}{\textbf{wiki}} \\ 
       \midrule
       \multirow{8}{*}{\rotatebox[origin=c]{90}{\textbf{Semantic}}} & \multicolumn{1}{l|}{\colorbox{white!20}{W2V (1 Host)}}  & \multicolumn{1}{c|}{\colorbox{blue!20}{75.86$\pm$0.07}}    & \multicolumn{1}{c|}{\colorbox{blue!20}{70.79$\pm0.54$}}    &  \multicolumn{1}{c}{\colorbox{blue!20}{79.10$\pm0.31$}} \\
                                          &\multicolumn{1}{l|}{\colorbox{white!20}{GEN (1 Host)}} & \multicolumn{1}{c|}{\colorbox{red!20}{  -0.22}}    &\multicolumn{1}{c|}{\colorbox{green!20}{ -0.22}}    &  \multicolumn{1}{c}{OOM} \\
                                          &\multicolumn{1}{l|}{\colorbox{white!20}{DMTK (1 Host)}} & \multicolumn{1}{c|}{\colorbox{red!20}{-13.79}} &\multicolumn{1}{c|}{\colorbox{red!20}{-18.43}}    &  \multicolumn{1}{c}{\colorbox{red!20}{-7.46}} \\
                                          &\multicolumn{1}{l|}{\colorbox{white!20}{DMTK(AVG) (32 Hosts)}} & \multicolumn{1}{c|}{\colorbox{red!20}{-57.36}}    &\multicolumn{1}{c|}{\colorbox{red!20}{-57.15}}    &  \multicolumn{1}{c}{\colorbox{red!20}{-34.39}} \\
 
                                          &\multicolumn{1}{l|}{\colorbox{white!20}{DMTK(GC) (32 Hosts)}} & \multicolumn{1}{c|}{\colorbox{red!20}{-10.93}}    &\multicolumn{1}{c|}{\colorbox{red!20}{-17}}    &  \multicolumn{1}{c}{\colorbox{red!20}{-5.17 }} \\
 
                                          &\multicolumn{1}{l|}{\colorbox{white!20}{{GW2V (1 Host)}}} & \multicolumn{1}{c|}{\colorbox{green!20}{+0.07}}    &\multicolumn{1}{c|}{\colorbox{green!20}{-0.08}}    &  \multicolumn{1}{c}{\colorbox{green!20}{+0.26}} \\
                                          &\multicolumn{1}{l|}{\colorbox{white!20}{{GW2V(AVG) (32 Hosts)}}} & \multicolumn{1}{c|}{\colorbox{red!20}{-7.00}}    &\multicolumn{1}{c|}{\colorbox{red!20}{-9.15}}    &  \multicolumn{1}{c}{\colorbox{red!20}{-4.03 }} \\

                                          &\multicolumn{1}{l|}{\colorbox{white!20}{{GW2V(GC) (32 Hosts)}}} & \multicolumn{1}{c|}{\colorbox{green!20}{+0.21}}    &\multicolumn{1}{c|}{\colorbox{green!20}{+0.07}}    &  \multicolumn{1}{c}{\colorbox{green!20}{-0.17}} \\
 
       \midrule
       \multirow{8}{*}{\rotatebox[origin=c]{90}{\textbf{Syntactic}}} & \multicolumn{1}{l|}{\colorbox{white!20}{W2V (1 Host)}}  & \multicolumn{1}{c|}{\colorbox{blue!20}{50.0$\pm0.18$}}    & \multicolumn{1}{c|}{\colorbox{blue!20}{50.0$\pm0.26$}}    &  \multicolumn{1}{c}{\colorbox{blue!20}{49.22$\pm0.12$}} \\
                                          &\multicolumn{1}{l|}{\colorbox{white!20}{GEN (1 Host)}}  & \multicolumn{1}{c|}{\colorbox{green!20}{ -0.14}}    &\multicolumn{1}{c|}{\colorbox{green!20}{ -0.12}}    &  \multicolumn{1}{c}{OOM} \\
 
                                          &\multicolumn{1}{l|}{\colorbox{white!20}{DMTK (1 Host)}} & \multicolumn{1}{c|}{\colorbox{red!20}{ -1.89}}&  \multicolumn{1}{c|}{\colorbox{red!20}{ -0.67 }}& \multicolumn{1}{c}{\colorbox{red!20}{ -3.11}}\\ 
                                          &\multicolumn{1}{l|}{\colorbox{white!20}{DMTK(AVG) (32 Hosts)}} & \multicolumn{1}{c|}{\colorbox{red!20}{-24.89}}&\multicolumn{1}{c|}{\colorbox{red!20}{-25.11}}&\multicolumn{1}{c}{\colorbox{red!20}{-23.11}}\\
 
                                          &\multicolumn{1}{l|}{\colorbox{white!20}{DMTK(GC) (32 Hosts)}} &\multicolumn{1}{c|}{\colorbox{red!20}{-3.56}}&\multicolumn{1}{c|}{\colorbox{red!20}{ -1.78}}&\multicolumn{1}{c}{\colorbox{red!20}{-1.44}}\\
 
                                          &\multicolumn{1}{l|}{\colorbox{white!20}{{GW2V (1 Host)}}} & \multicolumn{1}{c|}{\colorbox{red!20}{-0.37}}    &\multicolumn{1}{c|}{\colorbox{green!20}{0.0}}    &  \multicolumn{1}{c}{\colorbox{green!20}{ -0.12}} \\
                                          
                                          &\multicolumn{1}{l|}{\colorbox{white!20}{{GW2V(AVG) (32 Hosts)}}} & \multicolumn{1}{c|}{\colorbox{red!20}{ -4.89}}    &\multicolumn{1}{c|}{\colorbox{red!20}{ -4.11}}    &  \multicolumn{1}{c}{\colorbox{red!20}{ -7.55}} \\
 
                                          &\multicolumn{1}{l|}{\colorbox{white!20}{{GW2V(GC) (32 Hostss)}}} & \multicolumn{1}{c|}{\colorbox{green!20}{ +0.10}}    &\multicolumn{1}{c|}{\colorbox{green!20}{ +0.11}}    &  \multicolumn{1}{c}{\colorbox{green!20}{ +0.18}} \\
 
       \midrule
       \multirow{8}{*}{\rotatebox[origin=c]{90}{\textbf{Total}}} & \multicolumn{1}{l|}{\colorbox{white!20}{W2V (1 Host)}}  & \multicolumn{1}{c|}{\colorbox{blue!20}{72.36$\pm0.21$}}    & \multicolumn{1}{c|}{\colorbox{blue!20}{69.21$\pm0.42$}}    &  \multicolumn{1}{c}{\colorbox{blue!20}{74.10$\pm0.42$}} \\
                                          &\multicolumn{1}{l|}{\colorbox{white!20}{GEN (1 Host)}}  & \multicolumn{1}{c|}{\colorbox{green!20}{ +0.0}}    &\multicolumn{1}{c|}{\colorbox{green!20}{ -0.14}}    &  \multicolumn{1}{c}{OOM} \\

                                          &\multicolumn{1}{l|}{\colorbox{white!20}{DMTK (1 Host)}} & \multicolumn{1}{c|}{\colorbox{red!20}{-11.65}}& \multicolumn{1}{c|}{\colorbox{red!20}{-15.42}}& \multicolumn{1}{c}{\colorbox{red!20}{ -3.03}}\\ 
                                          &\multicolumn{1}{l|}{\colorbox{white!20}{DMTK(AVG) (32 Hosts)}}&\multicolumn{1}{c|}{\colorbox{red!20}{-51.29}}& \multicolumn{1}{c|}{\colorbox{red!20}{-51.71}}& \multicolumn{1}{c}{\colorbox{red!20}{-32.03}}\\
 
                                          &\multicolumn{1}{l|}{\colorbox{white!20}{DMTK(GC) (32 Hosts)}} &\multicolumn{1}{c|}{\colorbox{red!20}{ -9.86}}&\multicolumn{1}{c|}{\colorbox{red!20}{ -14.78}}&\multicolumn{1}{c}{\colorbox{red!20}{-5.24}}\\
 
                                          &\multicolumn{1}{l|}{\colorbox{white!20}{{GW2V (1 Host)}}} & \multicolumn{1}{c|}{\colorbox{green!20}{ -0.14}}    &\multicolumn{1}{c|}{\colorbox{green!20}{ -0.28}}    &  \multicolumn{1}{c}{\colorbox{green!20}{ +0.1}} \\
                                          
                                          &\multicolumn{1}{l|}{\colorbox{white!20}{{GW2V(AVG) (32 Hosts)}}} & \multicolumn{1}{c|}{\colorbox{red!20}{ -6.79}}    &\multicolumn{1}{c|}{\colorbox{red!20}{-9.28}}    &  \multicolumn{1}{c}{\colorbox{red!20}{-5.17}} \\
 
                                          &\multicolumn{1}{l|}{\colorbox{white!20}{{GW2V(GC) (32 Hostss)}}} & \multicolumn{1}{c|}{\colorbox{green!20}{ +0.14}}    &\multicolumn{1}{c|}{\colorbox{green!20}{ +0.29}}    &  \multicolumn{1}{c}{\colorbox{green!20}{ -0.17}} \\

      \bottomrule
     \end{tabular}
   \end{table}

  


We implement distributed \wordtovec and 
\nodetovec in our \graphanyvec framework, 
and we refer to these applications 
as \graphwordvec (GW2V) and \graphnodevec (GV2V) 
respectively. 
First, we compare these with the state-of-the-art 
third-party implementations (Section~\ref{subsec:compare-third-party}). 
We then analyze the impact of our Gradient Combiner 
(Section~\ref{subsec:eval-reduction}) 
and communication optimizations 
(Section~\ref{subsec:eval-comm-opt}). 
Our evaluation methodology is described in detail in 
Appendix~\ref{subsec:eval-methodology}.

\subsection{Comparing With The State-of-The-Art} 
\label{subsec:compare-third-party}

\noindent{\textbf{\wordtovecend:}} 
We compare GW2V with 
distributed-memory implementation DMTK~\cite{DMTK} 
and shared-memory implementations, W2V~\cite{word2vec2} 
and GEN~\cite{gensim}. 
Table~\ref{tbl:all_time} compares their 
training time on a single host. 
Figure~\ref{fig:res:speedup} shows the speedup 
of both GW2V and DMTK on 32 hosts over W2V on 1 host. 
Note that averaging (AVG) and 
our \modelcombiner (GC) methods are used to combine gradients 
during inter-host synchronization, 
so they have no impact on a single host.

{\it Performance:}
We observe that for all datasets on a single host, 
the training time of GW2V is similar to that of W2V, GEN, and DMTK. 
GW2V scales up to 32 hosts and speeds up the training time by 
$\sim13\times$ on average over 1 host.
In comparison with distributed DMTK on 32 hosts,
which uses parameter servers for synchronization
, GW2V is $\sim2\times$ faster on average for all datasets.
Figure~\ref{fig:res:speedup} also shows that 
there is negligible performance between using AVG and using GC 
to combine gradients in both DMTK and GW2V. 
Training wiki using GW2V takes only {\bf 1.9 hours}, 
which saves {\bf 18.6 hours} and {\bf 1.5 hours} 
compared to training using W2V and DMTK respectively. 

To understand the performance differences between DMTK and GW2V better, 
Figure~\ref{fig:res:dmtk_v3_stacked} shows the 
breakdown of their training time into 3 phases: 
inspection, computation, and (non-overlapped) communication. 
Firstly, GW2V's inspection phase as well as serialization and de-serialization during synchronization
are parallel using D-Galois~\cite{galois,gluon} parallel constructs and concurrent data-structures such 
as bit-vectors, work-lists, etc., whereas 
these phases are sequential in DMTK as it
uses non-concurrent data-structures such as set and vector provided by the C++ standard template library.
Moreover, in GW2V, 
hosts can update their masters in-place. 
This is not possible in DMTK as workers on each host 
have to fetch model parameters from servers on the same host 
to update, incurring overhead for additional copies.
Secondly, DMTK communicates much higher volume ($\sim3.5\times$) 
than \graphwordvecend.
GW2V memoizes the node IDs exchanged during inspection phase 
and sends only the updated values during broadcast and reduction. 
In contrast, DMTK sends the node IDs along with the updated 
values to the parameter servers 
during both broadcast and reduction. 
In addition, 
\graphwordvec inspection precisely identifies both the positive and negative 
samples required for the current round. 
DMTK, on the other hand, only identifies precise positive samples, 
and builds a pool for negative samples. 
During computation, negative samples are randomly picked from 
this pool. 
The entire pool is communicated from and to the 
parameter servers, 
although some of them may not be updated, 
leading to redundant communication.

\begin{table}
  \centering
    \caption{\nodetovec training time (sec) of DeepWalk on 1 host vs. \graphnodevec (GV2V) on 16 hosts.}
\label{tbl:3rdparty_time_node2vec}
\footnotesize
\begin{tabular}{@{ }l@{\hskip1.8pt}|r@{\hskip5pt}|r@{\hskip5pt}|r@{\hskip5pt}|r@{\hskip5pt}|r@{\hskip5pt}|r@{\hskip5pt}|r@{\hskip3pt}r@{\hskip5pt}r@{ }}
  \toprule
  \multicolumn{1}{c|}{\textbf{Dataset}}         & \multicolumn{1}{c|}{\textbf{DeepWalk}}  & \multicolumn{1}{c|}{\textbf{GV2V}} & \multicolumn{1}{c|}{\textbf{Speedup}} \\ 
\midrule
\multirow{1}{*}{\textbf{BlogCatalog}} & 115.3   & 28.8                & 4.0x                          \\
\multirow{1}{*}{\textbf{Flickr}}      & 976.7   & 183.1               & 5.3x                        \\
\multirow{1}{*}{\textbf{Youtube}}     & 11589.2 & 2226.2              & 5.2x                          \\

  \bottomrule
\end{tabular}

\end{table}

    \begin{table}[t]
      \footnotesize
      \centering
      
        \caption{\nodetovec accuracy (Macro F1 and Micro F1) of GV2V on 16 hosts relative to DeepWalk on 1 host.}
        \vspace{-10pt}
        \label{tbl:3rdparty_acc_node2vec}
      \begin{tabular}{l|l|l|rrrrrrrr}
      \toprule
       
          &        & \multicolumn{1}{c|}{{\begin{tabular}[l]{@{}l@{}}\textbf{\% Labeled}\\\textbf{Nodes}\end{tabular}}}   & \multicolumn{1}{c|}{\text{30\%}} & \multicolumn{1}{c|}{\text{60\%}} & \multicolumn{1}{c}{\text{90\%}} \\

        \midrule
          \multirow{4}{*}{\textbf{Micro-F1}}&\multirow{2}{*}{\text{BlogCatalog}} & \multicolumn{1}{l|}{\text{Deepwalk}} &   \multicolumn{1}{c|}{34.0}   &  \multicolumn{1}{c|}{37.2}         &  \multicolumn{1}{c}{38.4}   \\ 
                                             &                                  &   \multicolumn{1}{l|}{\colorbox{white!20}{\text{GV2V}}}      & \multicolumn{1}{c|}{\colorbox{white!20}{-0.1}}   & \multicolumn{1}{c|}{\colorbox{white!20}{+0.1}}         & \multicolumn{1}{c}{\colorbox{white!20}{+0.7}}      \\ 
  
        \cline{2-6}
                                             &\multirow{2}{*}{\text{Flickr}} & \multicolumn{1}{l|}{\text{Deepwalk}}     &  \multicolumn{1}{c|}{38.6}   &  \multicolumn{1}{c|}{40.4}         &  \multicolumn{1}{c}{41.1}   \\ 
                                             &                                   &   \multicolumn{1}{l|}{\colorbox{white!20}{\text{GV2V}}}      & \multicolumn{1}{c|}{\colorbox{white!20}{+0.1}}   & \multicolumn{1}{c|}{\colorbox{white!20}{-0.1}}         & \multicolumn{1}{c}{\colorbox{white!20}{-0.2}}     \\ 
  
        \midrule
          \multirow{4}{*}{\textbf{Macro-F1}}&\multirow{2}{*}{\text{BlogCatalog}} & \multicolumn{1}{l|}{\text{Deepwalk}}       &  \multicolumn{1}{c|}{34.1}   &  \multicolumn{1}{c|}{37.2}         &  \multicolumn{1}{c}{38.4}   \\ 
                                             &                                  &   \multicolumn{1}{l|}{\colorbox{white!20}{\text{GV2V}}}      & \multicolumn{1}{c|}{\colorbox{white!20}{-0.3}}   & \multicolumn{1}{c|}{\colorbox{white!20}{+0.1}}        & \multicolumn{1}{c}{\colorbox{white!20}{+0.7}}     \\ 
        \cline{2-6}
                                             &\multirow{2}{*}{\text{Flickr}} & \multicolumn{1}{l|}{\text{Deepwalk}}       &  \multicolumn{1}{c|}{26.5}   &  \multicolumn{1}{c|}{28.7}         &  \multicolumn{1}{c}{29.5}   \\ 
                                             &                                   &   \multicolumn{1}{l|}{\colorbox{white!20}{\text{GV2V}}}         & \multicolumn{1}{c|}{\colorbox{white!20}{+0.1}}   & \multicolumn{1}{c|}{\colorbox{white!20}{-0.1}}       & \multicolumn{1}{c}{\colorbox{white!20}{+0.3}}      \\ 
  
        \bottomrule
      \end{tabular}
      \end{table}

{\it Accuracy:} 
Table~\ref{tbl:3rdparty_acc} compares the accuracies (semantic, syntactic, and 
total) for all frameworks on 1 and 32 hosts 
relative to the accuracies achieved by W2V.
On a single host, GW2V
is able to achieve accuracies (semantic, 
syntactic and total) comparable to W2V.
DMTK on a single host is 
less accurate due to implementation differences in 
the Skip-gram model training; 
DMTK only updates learning rate between mini-batches, 
whereas others continuously degrade learning rate, 
and DMTK uses a different strategy to 
choose negative samples as described earlier. 
On 32 hosts,  
DMTK(AVG) has terrible accuracy and 
GW2V(AVG) has poor accuracy.
GC significantly improves the accuracies 
over AVG for both DMTK and GW2V. 
DMTK(GC) improves semantic by $\textbf{37.91\%}$, syntactic by $\textbf{22.09\%}$, and total
by $\textbf{34.79\%}$
to match its own single host accuracy.
GW2V(GC) improves all accuracies 
to match 
that of W2V. 

\noindent{\textbf{\nodetovecend:}} 
Table~\ref{tbl:3rdparty_time_node2vec} compares the training time of 
DeepWalk~\cite{deepwalk} on a single host 
with our GV2V 
on 16 hosts.
We observe that for all datasets, 
\graphnodevec can train the model $\sim4.8\times$ 
faster on average.
Similar to \graphwordvecend, 
this speedup does not come at the cost of the accuracy,
as shown in Table~\ref{tbl:3rdparty_acc_node2vec}, 
which shows the {\em Micro-F1} and 
{\em Macro-F1} score with 30$\%$, 60$\%$, and 90$\%$ labeled nodes.

\noindent{\textbf{Discussion:}} 
\graphanyvec significantly speeds up the training time 
for \wordtovec and \nodetovec applications 
by distributing the computation across the cluster
without sacrificing the accuracy.
Reduced training
time also accelerates the process of improving the training algorithms as it allows
application designers to make more end-to-end passes 
in a short duration of time.

\begin{figure*}
  \centering
  \begin{minipage}{0.28\textwidth}
    \centering
      \includegraphics[width=\textwidth]{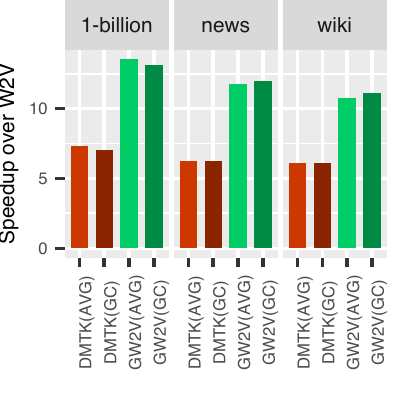}
  \vspace{-15pt}
  \caption{Speedup of DMTK and GW2V on 32 hosts over W2V on 1 host.}
  \label{fig:res:speedup}
\end{minipage}
\hspace{5pt}
\begin{minipage}{0.3\textwidth}
  \centering
    \includegraphics[width=\textwidth]{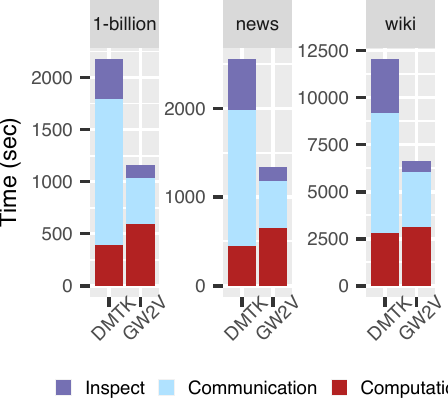}
\caption{Breakdown of training time of DMTK(GC) and GW2V(GC) on 32 hosts.}
  \label{fig:res:dmtk_v3_stacked}
\end{minipage}
\hspace{5pt}
\begin{minipage}{0.35\textwidth}
  \centering
    \includegraphics[width=0.9\textwidth]{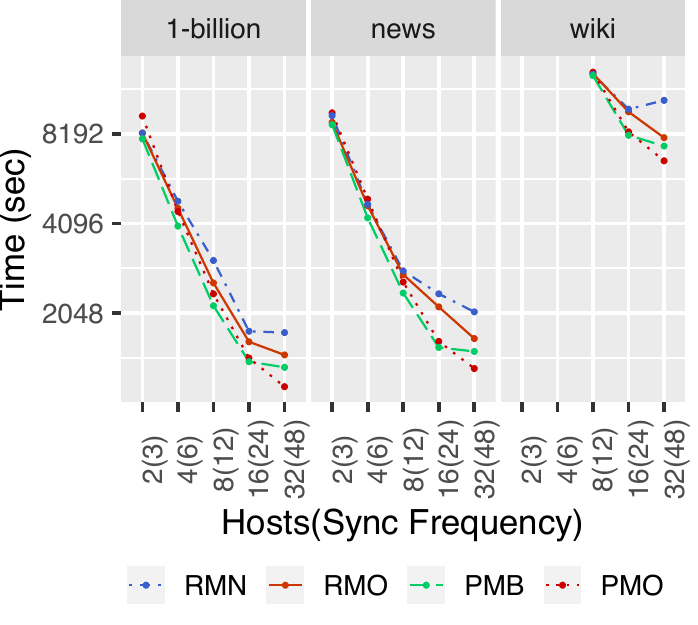}
    \vspace{-10pt}
\caption{Strong scaling of GW2V ({\small RMN: RepModel-Naive, RMO: RepModel-Opt, PMB: PullModel-Base,
  PMO: PullModel-Opt}).
}
\label{fig:res:gw2v_scaling}
\end{minipage}
\end{figure*}


\begin{figure*}
  \centering
\begin{minipage}{0.32\textwidth}
  \centering
    \includegraphics[width=0.96\textwidth]{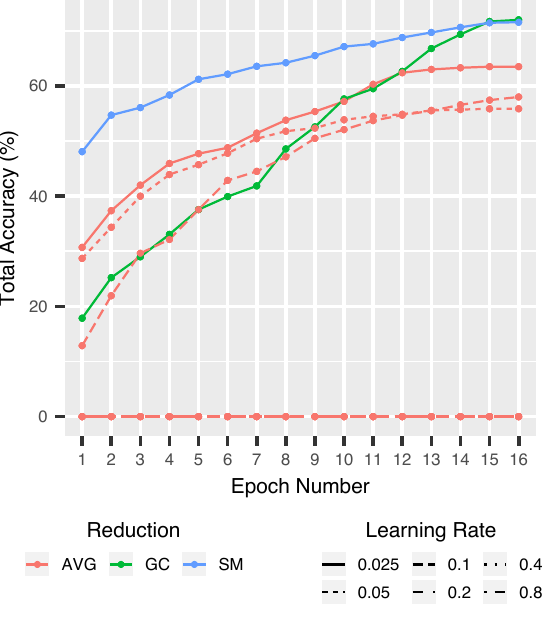}
    \vspace{-18pt}
    \caption{Accuracy of GW2V after each epoch for 1-billion 
    dataset on 1 host (SM) and on 32 hosts using GC and AVG.}
    \label{fig:res:mc_avg_verifyEachEpoch}
\end{minipage}
\hspace{5pt}
\begin{minipage}{0.32\textwidth}
  \centering
    \includegraphics[width=\textwidth]{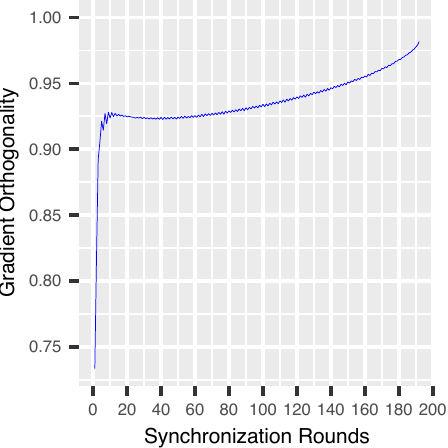}
\caption{GW2V for 1-billion dataset on 32 hosts: 
\% of gradients of a node that are orthogonal to each other 
in each epoch.}
\label{fig:res:grad_line_plot}
\end{minipage}
\hspace{5pt}
\begin{minipage}{0.32\textwidth}
  \centering
    \includegraphics[width=\textwidth]{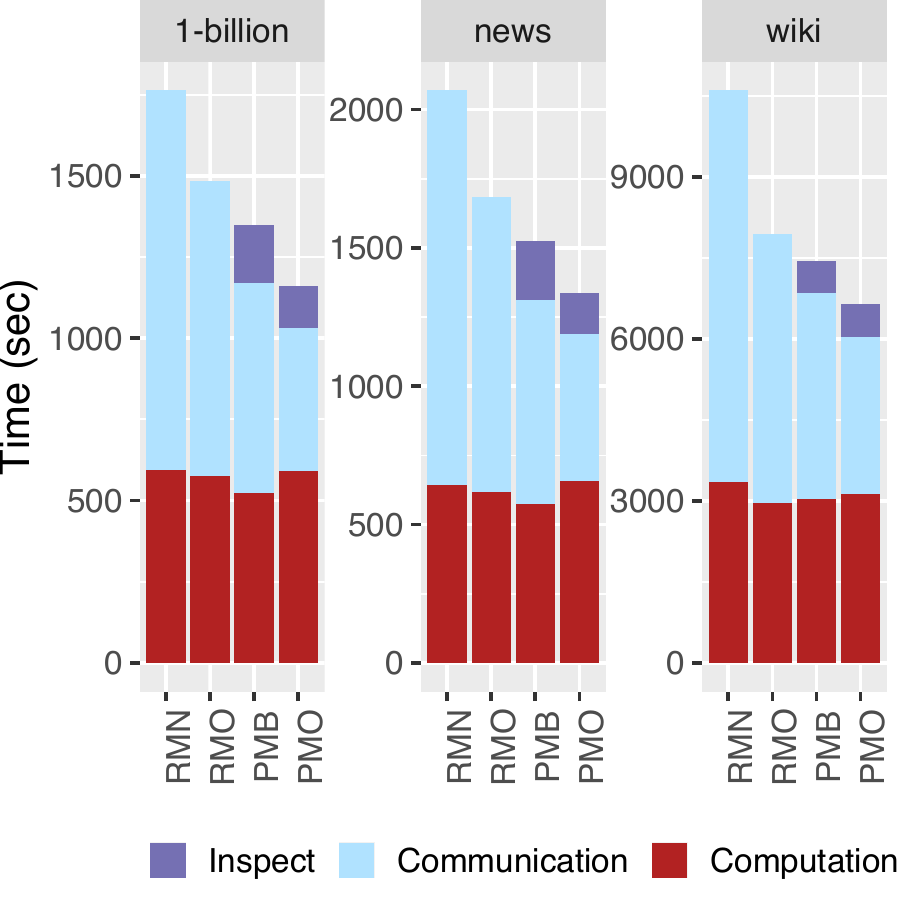}
\caption{Breakdown of training time of GW2V 
with different communication schemes
on 32 hosts.}
\label{fig:res:gw2v_stacked}
\end{minipage}
\end{figure*}


\subsection{Impact of Gradient Combiner (GC)}
\label{subsec:eval-reduction}

If time were not an issue, all machine learning algorithms would run
sequentially. A sequential SGD is simple to tune and converges fast.
Unfortunately, it is slow. A point $(x,y)$ on
Figure~\ref{fig:res:mc_avg_verifyEachEpoch} denotes the 
total accuracy ($y$) as a function of epoch ($x$). 
The blue line (SM) shows the
accuracy of GW2V on a single shared-memory host. 
It clearly converges to a high accuracy quickly. 
In contrast, the red lines plot accuracy of distributed GW2V 
that uses averaging the gradients (AVG) 
with different learning rates 
on 32 hosts.  
The learning rate of 0.025 is
the same as SM while the learning rate of 0.8 is 32 times
larger. 
The former converges slowly 
while the latter does not converge at all (accuracy is 0) 
because the learning rate is too large.
Finally, the green line plots accuracy of distributed GW2V 
that uses GC and 0.025 as the learning rate
on 32 hosts. 
GC has no problem meeting the
accuracy of the sequential algorithm. 
In addition to providing the
same accuracy as SM, it is $12\times$ times faster on 32 hosts
than SM. 
Not having to tune the learning rate and still getting
accuracy at scale is a significant qualitative contribution of our
work as tuning is a difficult task, in general.

In each round, we count the number of gradients that are 
being combined and the percentage of them that are orthogonal 
to each during combining. Figure~\ref{fig:res:grad_line_plot} 
shows this percentage as function of the rounds. 
In later rounds, more and more gradients are orthogonal to each other. 
As explained in Section~\ref{sec:combiner}, 
GC is more effective when the gradients are orthogonal. 
This is validated by the increase in accuracy for GC 
in later epochs in Figure~\ref{fig:res:mc_avg_verifyEachEpoch}.


\noindent{\textbf{Synchronization Rounds:}} 
\modelcombiner (GC) improves the accuracies significantly but in order 
to get accuracies comparable to shared-memory implementations,
the number of synchronization rounds in each epoch 
is an important knob to tune.
%
We observe that accuracies improve
as we increase the number of synchronization rounds within an epoch for both GC and AVG. 
Nonetheless, 
accuracies show more improvement for GC 
(for example, semantic: 3.07\%, syntactic: 3.99\% and total: 3.36\% when 
synchronization frequency is increased from 12 to 48 on 32 hosts) as 
opposed to AVG, which shows very little change in accuracies with synchronization rounds.
In general, we have observed that in order to maintain the desired accuracy, the synchronization 
frequency needs to be increased (roughly) linearly with the number of hosts. 
We have followed this rule of thumb in all our experiments.





\subsection{Impact of Communication Optimizations}
\label{subsec:eval-comm-opt}

Figure~\ref{fig:res:gw2v_scaling} shows the strong scaling of \graphwordvec 
with different communication optimizations 
(described in section~\ref{sec:comm-opt}): 
RepModel-Naive (RMN), RepModel-Opt (RMO),
PullModel-Base (PMB), and PullModel-Opt (PMO).
The 3 latter variants scale well up to 32 machines.

For 1-billion dataset, 
RepModel-Naive gives $4.7\times$ speedup on 32 hosts over 2 hosts. RepModel-Opt, which uses D-Galois
to only communicate the updated values for both reduction and broadcast, gives a speedup of $5.5\times$ by
reducing the communication volume. RepModel-Opt shows 16\% improvement over RepModel-Naive on 32 hosts, showing that 
RepModel-Opt is able to exploit the sparsity in the communication. The benefits of RepModel-Opt over RepModel-Naive
increases with the number of hosts for two main reasons: (a) synchronization frequency doubles with the number of hosts, 
thus communicating more data, and (b) as training data gets divided among hosts, sparsity in the updates increase. 

PullModel-Base not only avoids replication of the model on all hosts (thus can fit bigger models), but 
helps reduce communication volume over RepModel-Opt by only broadcasting model parameters 
required by hosts for the next round.
PullModel-Opt further reduces communication volume 
by taking into consideration the location of access as well: it only broadcasts
embedding vector for sources and training vector for destinations. 
These benefits come with an additional 
overhead of inspection phase before every 
synchronization round, but our evaluation shows that these overheads are offset by runtime improvements due to 
communication volume reduction. PullModel-Opt yields an average speedup of $8.1\times$ on 32 hosts over 2 hosts
and is $\sim 20.6\%$ on average faster than RepModel-Opt for all text datasets on 32 hosts.

Figure~\ref{fig:res:gw2v_stacked} shows the breakdown of the training time into inspection, computation, and
communication time of the variants on 32 hosts.
%
It is clear that all variants have similar computation time. 
RepModel-Opt communicates $\sim2\times$ less communication volume on average as opposed to RepModel-Naive, thus improving the
overall runtime. PullModel-Opt not only allows to train bigger models, it also further reduces communication volume by 
$\sim 11\%$ on average over RepModel-Opt by only broadcasting specific vectors to the proxies to be used in the next batch.


\noindent{\textbf{Summary:}} 
PullModel-Opt in \graphanyvec always performs better 
than the other variants by reducing the 
communication volume. 
These improvements are expected to grow as we scale to bigger datasets and number of hosts.
Hence, PullModel-Opt not only allows one to train bigger models, but also gives the best performance.

\section{Related Work}
\label{sec:related}
Many different types of models have been proposed in the past for estimating continuous representations of words, such
as Latent Semantic Analysis (LSA) and Latent Dirichlet Allocation (LDA). However, distributed representations of words
learned by neural networks are shown to perform significantly better than LSA~\cite{zhila2013combining,mikolov2013linguistic} and LDA is 
computationally very expensive on large data sets.
Mikolov et al.~\cite{word2vec1} proposed two simpler model architectures for computing 
continuous vector representations of words from very large unstructured data sets, known as Continuous Bag-of-Words (CBOW) and Skip-gram (SG).
These models removed the non-linear hidden layer and hence avoid dense matrix multiplications, which was responsible for most of the
complexity in the previous models.

CBOW is similar to the feedforward Neural Net Language Model (NNLM)~\cite{nnlm}, where the non-linear hidden layer is removed and the 
projection layer is shared for all words. All words get projected into the same position and their vectors are averaged.

SG on the other hand unlike CBOW, instead of predicting the current word based on the context, tries to maximize classification of a 
word based on another word within a sentence. Later Mikolov et al.~\cite{word2vec2} further introduced several extensions, such as
using hierarchical softmax instead of full softmax, negative sampling, subsampling of frequent words, etc., to SG model
that improves both the quality of the vectors and the training speed. 

Our work adapts the algorithm from this later work~\cite{word2vec2} for distribution. This work, together with many current implementations~\cite{gensim} 
are designed to run on a single machine but utilizing multi-threaded parallelism. Our work is motivated by the fact that these popularly used implementations take days 
or even weeks to train on large training corpus. Prior works on distributing \wordtovec either use synchronous data-parallelism~\cite{Deeplearning4j,Sparkword2vec,pword2vec} or
parameter-server style asynchronous data parallelism~\cite{DMTK}. However, they perform communication after every mini-batch, which is prohibitively expensive 
in terms of network bandwidth. Our design was motivated by the need to use commodity machines and network available on public clouds. Our approach communicates 
infrequently and uses our novel \modelcombiner to overcome the resulting staleness. 

Ordentlich et al.~\cite{Ordentlich16} propose a different method 
designed for models that do not fit in 
the memory of a single machine. They partition the model vertically with each machine containing part of the embedding and training vector for each word. 
These partitions compute partial dot products locally but communicate to compute global dot products. For all the publicly available benchmarks we could find,
the models fit in the memory in our machines. Nevertheless, our design allows for horizontal partitioning of large models if such a need arises in the future.

\section{Conclusions}
\label{sec:conclusions}
\graphanyvec substantially speeds up the training
time for \anytovec applications by distributing the
computation across the cluster without sacrificing the accuracy.
Reduced training time also accelerates the process of improving
the training algorithms as 
it allows application designers to make
more end-to-end passes in a short duration of time. 
\graphanyvec thus enables more explorations of \anytovec applications in areas such as 
natural language processing, network analysis, and code analysis. 

\modelcombiner is the key to \graphanyvecend's accuracy. 
It significantly improves accuracy of SGD compared to 
the traditional averaging method of combining gradients, 
even in third-party distributed implementations. 
\modelcombiner may also be useful in distributed training 
of other machine learning applications.

\bibliographystyle{abbrv}
\bibliography{graph,refs}

\clearpage
\appendix
\section{Appendix}
\subsection{Experimental Methodology}
\label{subsec:eval-methodology}

{\bf Hardware:}
All our experiments were conducted on the 
Stampede2
cluster at the Texas Advanced Computing Center using up to
32 Intel Xeon Platinum 8160 (``Skylake'') nodes, 
each with 48 cores with clock
rate 2.1Ghz, 192GB DDR4 RAM, and 32KB L1 data cache. 
Machines in the cluster are connected with a 100Gb/s Intel Omni-Path
interconnect. Code is compiled with g++ 7.1 and MPI mvapich2/2.3.

{\bf Datasets:}
Table~\ref{tbl:datasets} lists the training datasets used for our evaluation: text
datasets for \wordtovec and graph datasets for \nodetovecend. 
These datasets have different vocabulary
sizes (\# unique words or vertices), 
total training corpus size 
(\# occurrences of words or vertices), 
and sizes on disk. 
Prior 
\wordtovec and \nodetovec publications 
used the same datasets.
The wiki
(21GB) and Youtube
(2.8GB) datasets are the largest text  
and graph datasets respectively. 
We used DeepWalk~\cite{deepwalk}\footnote{\url{https://github.com/phanein/deepwalk}}
for generating training corpus for \nodetovec by performing 10 random walks each of length 40 from all 
vertices of the graph. 
We limit our \wordtovec and \nodetovec evaluation to 
32 and 16 hosts of Stampedes respectively 
because these datasets 
do not scale beyond that.
We report the accuracy and the training (or execution) time 
for all frameworks on these datasets, excluding preprocessing 
time, as an average of three distinct runs.


{\bf Shared-memory third-party implementations:}
We evaluated the Skip-gram~\cite{word2vec2} (with negative 
sampling) training model 
for both \wordtovec and \nodetovecend.
We compared \graphwordvec (GW2V) 
with the state-of-the-art shared-memory \wordtovec implementations, 
the original C implementation (W2V)~\cite{word2vec2} 
as well as the more recent Gensim (GEN)~\cite{gensim} python implementation. 
We also compared our \graphnodevec (GV2V) with the state-of-the-art shared-memory \nodetovec framework, 
DeepWalk~\cite{deepwalk}
(both DeepWalk
and Node2Vec~\cite{node2vec} use Gensim's~\cite{gensim} Skip-gram model).

{\bf Distributed-memory third-party implementations:}
We compared \graphwordvec with the state-of-the-art distributed-memory \wordtovec 
from Microsoft's Distributed Machine Learning Toolkit (DMTK)~\cite{DMTK}, 
which is based on the parameter server model. 
The model is distributed among parameter server hosts. 
During 
execution, hosts acting as workers 
request the required model parameters from the servers and send model updates 
back to the servers. 
Each host in the cluster acts as both server and worker, 
and it is the only configuration possible. 
DMTK uses OpenMP for parallelization within a host 
(\graphwordvec uses Galois~\cite{galois} for parallelization within a host). 
Both \graphwordvec and DMTK use MPI for communication between hosts. 
We modified DMTK to include a runtime option of configuring 
the number of synchronization rounds. 
DMTK uses averaging as the reduction operation to combine 
the gradients. We refer to this as DMTK(AVG). 
We also implemented our \modelcombiner in DMTK and we 
call this DMTK(GC).
There are no prior distributed implementations of \nodetovecend. 
Unless otherwise specified, 
GW2V and GV2V uses GC to combine gradients and 
use PullModel-Opt communication optimization.

{\bf Hyper-parameters:}
We used the hyper-parameters suggested by~\cite{word2vec2}, 
unless otherwise specified: 
window size of 5, number of negative samples of 15, 
sentence length of 10K, 
threshold of $10^{-4}$ for \wordtovec and $0$ for 
\nodetovec for down-sampling the frequent words, and 
vector dimensionality 
$N$ of 200. 
All models were trained for 16 epochs. 
For distributed frameworks, GV2V, GW2V, and DMTK, we compared 2 gradient combining methods: 
Averaging (AVG) (default for distributed training of \anytovec applications) and 
our novel \modelcombiner (GC) method. 
Unless otherwise specified, 
GV2V, GW2V, and DMTK use the same number of synchronization 
rounds: 1 for 1 host, 3 for 2 hosts, 6 for 4 hosts, 
12 for 8 hosts, 24 for 16 hosts, and 48 for 32 hosts. 
Note that the default for DMTK is 1 synchronization round for 
any number of hosts, but this yields very low accuracy, 
so we do not report these results.


{\bf Accuracy:}
In order to measure the accuracy of trained models of \wordtovec on different datasets, we used the analogical 
reasoning task outlined by original \wordtovecend~\cite{word2vec2} paper.
We evaluated the accuracy using scripts and question-words.txt provided by 
the \wordtovec code base\footnote{\url{https://github.com/tmikolov/word2vec}}. Question-words.txt consists of 
analogies such as "Athens" : "Greece" :: "Berlin" : ?, which are predicted by finding a vector $x$ such that embedding vector($x$)
is closest to embedding vector("Athens") - vector("Greece") + vector("Berlin") according to the cosine distance. For this particular
example the accepted value of $x$ is "Germany". There are 14 categories of such questions, which are broadly divided into 2 main
categories: (1) the syntactic analogies (such as "calm" : "calmly" :: "quick" : "quickly") and (2) the semantic analogies such
as the country to capital city relationship. 
We report semantic, syntactic, and total accuracy averaged over all the 14 categories of questions.
For \nodetovec, we measured {\em Micro-F1} and {\em Macro-F1} scores using scoring 
scripts provided by DeepWalk~\cite{deepwalk}\footnote{\url{https://github.com/phanein/deepwalk/blob/master/example_graphs/scoring.py}}.



\subsection{Proof of Properties of \modelcombiner}
\label{sec:prop}
The three properties of \modelcombiner are (1) $g_1^T\cdot g_1^O \geq 0$, (2) $\norm{g_1^O}\leq\norm{g_1}$, and (3) $g_2^T\cdot g_1^O=0$
where $g_1^O=g_1-\frac{g_2^T\cdot g_1}{\norm{g_2}^2}g_2$. For the proof, assume that the angle between $g_1$ and $g_2$ is $\theta$.

Property (1): $g_1^T\cdot g_1^O \geq 0$. 
\begin{equation}
\begin{split}
g_1^T\cdot g_1^O&=g_1^T\cdot (g_1-\frac{g_2^T\cdot g_1}{\norm{g_2}^2}g_2)=\norm{g_1}^2-\frac{g_2^T\cdot g_1}{\norm{g_2}^2}g_1^T\cdot g_2\\
		&=\norm{g_1}^2 - \frac{\norm{g_1}^2\norm{g_2}^2\cos^2\theta}{\norm{g_2}^2}=\norm{g_1}^2\sin^2\theta\geq0
\end{split}
\end{equation}

Property (2): $\norm{g_1^O}\leq\norm{g_1}$.
\begin{equation}
\begin{split}
	\norm{g_1^O}^2 &= \norm{g_1-\frac{g_2^T\cdot g_1}{\norm{g_2}^2}g_2}^2\\
	               &=\norm{g_1}^2 + \frac{(g_2^T\cdot g_1)^2}{\norm{g_2}^4}\norm{g_2}^2-2\frac{g_2^T\cdot g_1}{\norm{g_2}^2}g_1^T\cdot g_2\\
		       &=\norm{g_1}^2 + \frac{(g_2^T\cdot g_1)^2}{\norm{g_2}^2}-2\frac{(g_2^T\cdot g_1)^2}{\norm{g_2}^2} =\norm{g_1}^2 - \frac{(g_2^T\cdot g_1)^2}{\norm{g_2}^2} \\
		       &=\norm{g_1}^2 - \frac{\norm{g_1}^2\norm{g_2}^2\cos^2\theta}{\norm{g_2}^2}= \norm{g_1}^2-\norm{g_1}^2\cos^2\theta\\
		       &=\norm{g_1}^2\sin^2\theta\leq\norm{g_1}^2
\end{split}
\end{equation}

Property (3): $g_2^T\cdot g_1^O=0$.
\begin{equation}
\begin{split}
	g_2^T\cdot g_1^O &= g_2^T\cdot (g_1-\frac{g_2^T\cdot g_1}{\norm{g_2}^2}g_2)=g_2^T\cdot g_1 - \frac{g_2^T\cdot g_1}{\norm{g_2}^2}\norm{g_2}^2\\
			 &=g_2^T\cdot g_1 - g_2^T\cdot g_1 = 0
\end{split}
\end{equation}

\subsection{\modelcombiner (GC) Convergence Proof}
\label{sec:combiner-proof}
\cite{conv} discusses the requirements for a training algorithm to converge to its optimal answer. Here we will present a simplied version of
Theorem 1 and Corollary 1 from~\cite{conv}.

Suppose that there are $N$ training examples for a model with loss functions $L_1(w),\dots,L_N(w)$ where $w$ is the model
parameter and $w_0$ is the initial model.  
Define $L(w)=\frac{1}{N}\sum_i L_i(w)$. Also assume that $w^*$ is the optimal model where $L(w^*)\leq L(w)$ for all $w$s. A training algorithm is {\em pseudogradient} if:
\begin{itemize}
	\item It is an iterative algorithm where $w_{i+1} = w_i-\alpha_i h_i$ where $h_i$ is a random vector and $\alpha_i$ is a scalar.
	\item $\forall \epsilon \exists \delta: E(h_i)^T\cdot \nabla L(w)\geq \delta>0$ where $L(w)\geq L(w^*)+\epsilon$ and $w^*$ is the optimal model.
	\item $E(\norm{h_i}^2)<C$ where $C$ is a constant.
	\item $\forall i: \alpha_i\geq 0$, $\sum_i \alpha_i=\inf$, and $\sum_i \alpha_i^2 < \inf$.
\end{itemize}

The following Theorem is taken from~\cite{conv}.
\begin{theorem}\label{conv-theorem}
A pseudogradient training algorithm converges to the optimal model $w^*$.
\end{theorem}

As a reminder, $GC(g_1,g_2) = g_2+g_1-\frac{g_2^T\cdot g_1}{\norm{g_2}^2}g_2$ and for $k$ gradients,
$GC(g_1,\dots,g_k)=GC(g_k,\dots,GC(g_3,GC(g_2,g_1))\dots)$. Suppose 
$G(w_i)=\{\frac{\partial L_1}{\partial w}|_{w_i},\dots,\frac{\partial L_N}{\partial w}|_{w_i}\}$ is a random variable
distribution of the gradients at $w_i$. 
\begin{theorem}
Suppose $h_i=GC(g_1,\dots,g_k)$ where $g_1,\dots,g_k$ are $k$ independently chosen gradients from $G(w_i)$.
$h_i$ is pseudogradient.
\end{theorem}
\begin{proof}
To facilitate the proof of the pseudogradient properties of $h_i$, we rewrite GC formula as follows:
\begin{equation}
\begin{split}
	GC(g_1,g_2)&=g_2+g_1-\frac{g_2^T\cdot g_1}{\norm{g_2}^2}g_2=g_1 + g_2 - \frac{g_2\cdot g_2^T}{\norm{g_2}^2}g_1\\
	&=\Big(I-\frac{g_2\cdot g_2^T}{\norm{g_2}^2}\Big)g_1+g_2
\end{split}
\end{equation}
where $\frac{g_2\cdot g_2^T}{\norm{g_2}^2}$ is a rank-1 matrix. 

First by induction, we prove that $E(GC(g_1,\dots,g_k))$ and $\nabla L(w_i)$ have a positive inner product.

Base of the induction:
because $g_1$ and $g_2$ are independently chosen, $E(GC(g_1,g_2))$ can be calculated by:
\begin{equation}
\begin{split}
E(GC(g_1,g_2)) &= E\Big(I-\frac{g_2\cdot g_2^T}{\norm{g_2}^2}\Big)E(g_1) + E(g_2)\\
							 &=\frac{1}{N}\sum_{a\in G(w_i)}\Big(I-\frac{a\cdot a^T}{\norm{a}^2}\Big)\nabla L(w_i) + \nabla L(w_i)
\end{split}
\end{equation}
where the last equation comes from the fact that $E(g_j)=\nabla L(w_i)$ for any $g_j\in G(w_i)$. Using the above formula,
we have:
\begin{equation}
\vspace{-10pt}
\begin{split}
&\nabla L(w_i)^T \cdot E(GC(g_1,g_2)) = \nabla L(w_i)^T \frac{1}{N}\sum_{a\in G(w_i)}\Big(I-\frac{a\cdot a^T}{\norm{a}^2}\Big)\nabla L(w_i)\\
&+ \nabla L(w_i)^T \cdot\nabla L(w_i)=\frac{1}{N}\sum_{a\in G(w_i)}\Big(\norm{\nabla L(w_i)}^2\\
&-\nabla L(w_i)^T\frac{a\cdot a^T}{\norm{a}^2}\nabla L(w_i)\Big) +\norm{\nabla L_{w_i}}^2\\
&=\frac{1}{N}\sum_{a\in G(w_i)}\Big(\norm{\nabla L(w_i)}^2-\norm{\nabla L(w_i)}^2\cos^2\theta_a\Big) + \norm{\nabla L_{w_i}}^2 \\
&=\frac{1}{N}\sum_{a\in G(w_i)}\Big(\norm{\nabla L(w_i)}^2\sin^2\theta_a\Big) + \norm{\nabla L_(w_i)}^2 \geq \norm{\nabla L_(w_i)}^2
\end{split}
\end{equation}
where $\theta_a$ is the angle between $a$ and $\nabla L(w_i)$.

Induction step:
Now assume that $\nabla L(w_i)^T \cdot GC(g_1,\dots,g_{l-1})\geq \norm{\nabla L_(w_i)}^2$. Also, assume that the
random vector distribution that is generated by $GC(g_1,\dots,g_{l-1})$ is $X$ with a size of $M$. Therefore, for the induction step, we have:
\begin{equation}
\vspace{-10pt}
\begin{split}
&\nabla L(w_i)^T \cdot E(GC(g_1,\dots,g_{l-1},g_l)) = \nabla L(w_i)^T \cdot \Bigg(E(g_l)\\
&+ E(GC(g_1,\dots,g_{l-1})) - \frac{1}{M}\sum_{a\in X} \frac{a\cdot a^T}{\norm{a}^2}E(g_1)\Bigg)=\norm{\nabla L(w_i)}^2\\
&+\nabla L(w_i)^T\cdot E(GC(g_1,\dots,g_{l-1}))-\frac{1}{M}\sum_{a\in X} \nabla L(w_i)^T\frac{a\cdot a^T}{\norm{a}^2}\nabla L(w_i) \\
&= \nabla L(w_i)^T\cdot E(GC(g_1,\dots,g_{l-1}))+\norm{\nabla L(w_i)}^2\\
&-\frac{1}{M}\sum_{a\in X} \norm{\nabla L(w_i)}^2\cos^2\theta_a =\nabla L(w_i)^T\cdot E(GC(g_1,\dots,g_{l-1}))\\
&+ \frac{1}{M}\sum_{a\in X} \norm{\nabla L(w_i)}^2\sin^2\theta_a \geq \norm{\nabla L(w_i)}^2
\end{split}
\end{equation}
where $\theta_a$ is the angle between $\nabla L(w_i)$ and $a$ and the last inequality is implied by the induction assumption. Therefore, the inner product
of GC and $\nabla L(w_i)$ is positive in expectation as $\norm{\nabla L(w_i)}^2$ is only zero at the optimal point.

Now we prove that norm of GC is bounded. We assume that norm of the gradients of $G(w_i)$ for all $w_i$s are bounded.
\begin{equation}
\norm{GC(g_1,g_2)}^2=\norm{g_1^O+g_2}^2=\norm{g_1^O}^2+\norm{g_2}^2\leq \norm{g_1}^2+\norm{g_2}^2
\end{equation}
where the last equality is implied by the Pythagorean theorem (property (3): $g_1^O$ and $g_2$ are orthogonal) and 
the inequality is implied by $\norm{g_1^O}\leq \norm{g_1}$ (property (2)). This can be easily extended to GC
for $k$ gradients: $\norm{GC(g_1,\dots,g_k)}^2\leq \sum_j \norm{g_j}^2$. Therefore, $E(\norm{GC(g_1,\dots,g_k)}^2)\leq \sum_j E(\norm{g_j}^2)=kE(\norm{g_1})^2$.
Therefore, norm of GC for $k$ gradients is also bounded.

Given that
 GC uses the same learning rate schedule as sequential SGD, the requirement for the learning rate of a pseudogradient training
algorithm is already met. 
Thus, from Theorem~\ref{conv-theorem}, 
the gradients computed with GC moves the model to
the optimal point.
\end{proof}



\end{document}